\documentclass[twoside]{article}

\usepackage[preprint]{aistats2027}

\usepackage[numbers,square]{natbib}

\usepackage[utf8]{inputenc}
\usepackage[T1]{fontenc}
\usepackage{amsmath,amssymb,amsthm,mathtools}
\usepackage{bm}
\usepackage{booktabs}
\usepackage{graphicx}
\usepackage{xcolor}
\usepackage{algorithm}
\usepackage{algpseudocode}
\usepackage{tikz}
\usetikzlibrary{arrows.meta,shapes.misc,decorations.pathreplacing}
\definecolor{pushc}{HTML}{1D9E75}   %
\definecolor{pullc}{HTML}{D85A30}   %
\usepackage[hidelinks]{hyperref}

\newtheorem{proposition}{Proposition}
\newtheorem{corollary}{Corollary}

\newcommand{\ours}{\textsc{SDRE}}                 %
\newcommand{\R}{\mathbb{R}}
\newcommand{\E}{\mathbb{E}}
\newcommand{\KL}{\mathrm{KL}}
\newcommand{\A}{\mathcal{A}}                      %
\newcommand{\Lv}{\mathcal{L}}                     %
\newcommand{\vf}{\bm{v}}
\newcommand{\x}{\bm{x}}

\newcommand{\s}{\bm{s}}
\newcommand{\T}{\top}

\newcommand{\papertitle}{Density Ratio Estimation with Stein Displacement Fields}

\hypersetup{pdftitle={Density Ratio Estimation with Stein Displacement Fields}, pdfauthor={Song Liu}}

\begin{document}

\twocolumn[
\aistatstitle{\papertitle}
\aistatsauthor{Song Liu}
\aistatsaddress{University of Bristol} ]

\begin{abstract}

Density ratios quantify distribution shift from a probability-mass point of view, whereas displacement fields describe, from a
dynamical point of view, how one distribution is transported onto another. Although both offer complementary insights, they are
usually estimated separately, and converting one into the other requires post-processing. In this paper, we estimate
the density ratio between a target and a base distribution by parametrizing it through a displacement field acting on the base: the
log-ratio is modeled as minus the Stein operator of the base applied to the field, up to a normalizing constant.
This gives both statistical and dynamical descriptions of the distribution shift through a single convex optimization problem.
Iterating this
estimate-and-move step gives two inference algorithms: \emph{push-forward} moves the model and corrects a
pretrained sampler without retraining it, whereas \emph{pull-back} moves the data closer to the base and fits a transformation model one
layer at a time.
Applications to distribution shift in simulation-based inference and to nonlinear independent component analysis illustrate the
benefits and limitations of the approach.
\end{abstract}

\section{Introduction}
\label{sec:intro}

Many machine learning problems come down to comparing two distributions. 
Broadly speaking, there are two different ways to do so.
The first is
information-theoretic: a divergence, such as the Kullback--Leibler (KL) divergence or another $f$-divergence, measures how 
two distributions spread their probability mass differently. These divergences serve as training objectives for generative adversarial networks
\citep{goodfellow2014generative,nowozin2016f}, and their gradient flows move one distribution towards the other \citep{jordan1998variational,liu2016stein,lu2019accelerating,maurais2024sampling}. An $f$-divergence is an expectation of a function of the
density ratio $r=p/q$ \citep{nguyen2010estimating,yi2023monoflow,liu2024minimizing}. Density ratio
estimation (DRE) has therefore become a building block of this approach, as well as a tool in its own right for importance weighting,
two-sample testing and likelihood-free inference \citep{sugiyama2012density,cranmer2020frontier,hermans2020likelihood}. However, ratio estimation is often fragile. When $p$ and $q$ are far apart, samples from one say little about the other, a problem known as the
``density chasm'' \citep{rhodes2020telescoping}, 
and adversarial training is notoriously unstable \citep{arjovsky2017towards}.

The second approach compares distributions through transport. Rather than comparing densities, it learns how to move one distribution onto the other with a
displacement field or a flow map. This view underlies optimal-transport domain adaptation \citep{courty2017optimal} and flow-based
generative modeling, where flow matching, rectified flows and stochastic interpolants fit the velocity of a prescribed path by
least-squares regression \citep{lipman2023flow,liu2023flow,albergo2023building}. These methods fit the velocity by regression, avoiding adversarial ratio estimation. Their prescribed paths, however, need not follow the gradient flow of a specified divergence.

Can we learn a divergence and, at the same time, a displacement field along which moving one distribution decreases
that divergence?
We show that this is feasible when the score of the \emph{base} distribution $q$ can be evaluated and the \emph{target} distribution $p$ is observed through samples.
Stein's method connects the two views \citep{stein1972bound,anastasiou2023stein}. The Langevin--Stein operator $\A_q$ of the base maps a vector field
$\vf:\R^d\to\R^d$ to a scalar function $\A_q\vf$ that has zero mean under $q$ and depends on $q$ only through its score
$\s_q=\nabla\log q$ (Section~\ref{sec:background:stein}). We model the log-ratio as
$-\A_q\vf$, up to a constant. 
This modelling choice has two consequences. First, 
a convex optimization problem derived from the Donsker--Varadhan (DV) lower bound estimates both $p/q$ and a lower bound on $\KL(q\,\|\,p)$ using only samples of $p$ and the score of $q$.
Similar Stein-parametrized ratio
models have been used to fit intractable statistical models \citep{liu2019fisher}. Second, moving $q$ along the same field,
$\x\mapsto\x+\eta\vf(\x)$ with a small step size $\eta>0$, changes its log-density by $-\eta\,\A_q\vf$ to first order.
At a population optimum attaining the DV bound, this is equivalent, up to normalization and to first order, to multiplying $q$ by $r^{\eta}$.
\textbf{The field that exactly represents the log-ratio therefore also gives a displacement towards $p$.}
Its infinitesimal density change agrees with the Fisher--Rao gradient flow of $\KL(q\,\|\,p)$, realized through transport.

This infinitesimal interpretation motivates repeating the refinement step, and there are two natural approaches. If $q$ is a pretrained generative model with a tractable density (e.g., flow models), we can move the \emph{model} by appending displacement layers to the model, which correct the model without retraining it. If instead $q$ belongs to a simple parametric family that we want to
fit, we can move the \emph{data}, pulling them back towards the family through each displacement. 
We call the two algorithms \emph{push-forward} and \emph{pull-back} and show that they are naturally suited to model updating and latent independent source recovery tasks respectively.

In this paper, we first propose Stein displacement ratio estimation (\ours{}), a convex density ratio estimator that estimates not
only the ratio $p/q$ but also a displacement field.
When the population DV bound is attained, this field induces, to first order, a density update along the Fisher--Rao gradient flow of
$\KL(q\,\|\,p)$ (Section~\ref{sec:method}).
This result motivates the two inference
algorithms of Section~\ref{sec:alg}. Push-forward appends displacement layers to a pretrained sampler to fit the data while keeping
its model frozen, and pull-back transforms the data through an invertible flow towards a prior from a simple parametric family.
Finally, we apply both algorithms to two problems
(Section~\ref{sec:app}). Push-forward adapts a pretrained diffusion posterior sampler to a new posterior target,
performing on par with discriminator guidance while keeping an exact density. 
Pull-back refines pretrained ICA models and gives strong signal identification in
synthetic and speech settings.
We also
discuss the limitations of the approach (Section~\ref{sec:exp:limits}).

\section{Background}
\label{sec:background}

Our estimator combines two ingredients: density ratio estimation from samples (Section~\ref{sec:background:dre}) and Stein's method
(Section~\ref{sec:background:stein}).

\subsection{Density ratio estimation}
\label{sec:background:dre}

We first recall how a density ratio is estimated when samples of both distributions are available; our estimator in
Section~\ref{sec:method} modifies this construction. Suppose we have samples from two distributions,
\[\{\x_i\}_{i=1}^n\sim p,\qquad \{\tilde\x_j\}_{j=1}^{\tilde n}\sim q.\]
The KL importance estimation procedure \citep[KLIEP;][]{sugiyama2008direct} models the ratio $p/q$ by
$r_{\bm\alpha}(\x)=\bm\alpha^\T\bm\phi(\x)$ with non-negative basis functions and solves the constrained convex problem
\begin{align}
\label{eq.linear.kliep}    
\max_{\bm\alpha} \frac{1}n\sum_i\log r_{\bm\alpha}(\x_i), ~~\text{s.t.}~~ \bm\alpha\ge\bm0, \frac{1}{\tilde n}\sum_j r_{\bm\alpha}(\tilde\x_j)=1.
\end{align} 
The log-linear variant $r_{\bm\alpha}=\frac{\exp(\bm\alpha^\T\bm\phi)}{\frac{1}{\tilde n}\sum_j\exp(\bm\alpha^\T\bm\phi(\tilde\x_j))}$
removes the constraints and gives the unconstrained convex problem \citep{tsuboi2009direct}
\begin{equation}
  \max_{\bm\alpha}\;\; \frac{1}n\sum_{i=1}^n\bm\alpha^\T\bm\phi(\x_i)\;-\;\log\frac{1}{\tilde n}\sum_{j=1}^{\tilde n}e^{\bm\alpha^\T\bm\phi(\tilde\x_j)} .
  \label{eq:kliep}
\end{equation}
This log-linear KLIEP closely resembles the Donsker--Varadhan (DV) representation of the KL divergence \citep{donsker1983asymptotic},
\begin{equation}
  \KL(p\,\|\,q)\;=\;\sup_f\;\E_p[f]-\log\E_q\big[e^{f}\big]. 
  \label{eq:dv}
\end{equation}
Restricting $f$ to a log-ratio model that is linear in its parameters and replacing expectations with sample averages gives the log-linear KLIEP above.

When $p$ and $q$ are far apart, the samples of one say little about the other \citep[the ``density chasm'';][]{rhodes2020telescoping};
recent methods bridge the gap with classifiers over intermediate or auxiliary distributions
\citep{rhodes2020telescoping,srivastava2023estimating}, time score matching along a continuous path
\citep{choi2022density,yu2025density}, or a shared invertible feature map \citep{choi2021featurized}. All of these methods, like
\eqref{eq:kliep} and \eqref{eq:dv}, need samples from \emph{both} distributions.

\subsection{Stein's method}
\label{sec:background:stein}

Stein's method characterizes a distribution by an operator whose expectation under that distribution vanishes
\citep{stein1972bound,anastasiou2023stein}. For a smooth, positive density $q$ on $\R^d$ with score $\s_q=\nabla\log q$, the
Langevin--Stein operator acts on a vector field $\vf:\R^d\to\R^d$ as
\begin{equation}
  \A_q\vf \;=\; \s_q^\T\vf + \nabla\!\cdot\vf \;=\; \frac{\nabla\!\cdot(q\vf)}{q} .
  \label{eq:stein-op}
\end{equation}
By the second form and the divergence theorem,
\begin{equation}
  \E_q[\A_q\vf]\;=\;\int\nabla\!\cdot(q\vf)\,\mathrm d\x\;=\;0
  \label{eq:stein}
\end{equation}
whenever $q\vf$ vanishes sufficiently fast at infinity and $\A_q\vf$ is integrable; such fields form the \emph{Stein class} of $q$, and
\eqref{eq:stein} is \emph{Stein's identity}. The operator depends on $q$ only through its score, so it can be evaluated when the normalizing constant of $q$
is unknown.

Optimizing the expectation of $\A_q\vf$ under a data distribution over a class of fields gives Stein discrepancies \citep{gorham2015measuring}, used for
goodness-of-fit tests through the closed-form kernelized Stein discrepancy (KSD) \citep{liu2016kernelized,chwialkowski2016kernel}. 
Stein variational gradient descent \citep[SVGD;][]{liu2016stein} moves particles from a
distribution $\rho$ towards a target with known score $\s_q$: under the map $\x\mapsto\x+\epsilon\vf(\x)$, the derivative of
$\KL(\rho\,\|\,q)$ at $\epsilon=0$ is $-\E_\rho[\A_q\vf]$, so maximizing $\E_\rho[\A_q\vf]$ over the kernel ball gives the steepest
descent direction within that class.

In these uses, the scalar function $\A_q\vf$ measures a discrepancy through its expectation or enters the directional derivative of the KL divergence.
The vector field $\vf$, selected by maximizing $\E_\rho[\A_q\vf]$ over the kernel ball, provides the particle descent direction in SVGD.

\subsection{Stein density ratio estimation}
\label{sec:background:sdre}

Estimating the ratio between a data distribution and \emph{an unnormalized model} is useful in statistical inference, and the Stein operator can help in this setting.
\citet{liu2019fisher} estimate $r=p/q$ where $q$ is
an unnormalized model with a known score and $p$ is observed through samples. They take $b$ stacked scalar feature functions $\bm f:\R^d\to\R^b$, apply
the Stein operator to their gradient fields, and use the linear ratio model $r_{\bm\delta}=1+\bm\delta^\T\A_q\nabla\bm f$, where
$\A_q\nabla\bm f$ collects $\A_q\nabla f_k$ for $k=1,\dots,b$. By Stein's identity, $\E_q[r_{\bm\delta}]=1$ for every $\bm\delta$, so the
normalization constraint of KLIEP holds automatically at the population level. The normalization constraint no longer requires samples from $q$, leaving the objective
\begin{equation}
  \max_{\bm\delta}\;\; \frac{1}n\sum_{i=1}^n\log\big(1+\bm\delta^\T(\A_q\nabla\bm f)(\x_i)\big).
  \label{eq:sdre-linear}
\end{equation}
The logarithms restrict the domain to positive fitted ratios at the samples. The optimized value serves as an estimate of $\KL(p\,\|\,q)$.
\citet{liu2019fisher} minimize this estimate over the parameters of $q$ to fit intractable
models. The linear model is not guaranteed to be positive away
from the samples, 
which limits its use for out-of-sample ratio evaluation.

\section{Stein Displacement Ratio Estimation}
\label{sec:method}

Let $q$ be a \emph{base distribution} on $\R^d$, for example a prior or a pretrained generative
model, and let $p$ be a \emph{target}. We can evaluate the score of the base, $\s_q=\nabla\log q$, which does not require its normalizing
constant, and we observe the target only through samples $\x_1,\dots,\x_n$. Our goal is to estimate the density ratio $r=p/q$. Like the
density ratio estimator in Section~\ref{sec:background:sdre}, our estimator uses no samples from $q$. Throughout, $q$ is smooth and positive, and every vector field $\vf:\R^d\to\R^d$ below belongs to the Stein class of $q$ (Section~\ref{sec:background:stein}), so that
Stein's identity \eqref{eq:stein} holds.

Consider the population objective of the log-linear KLIEP
for the reciprocal ratio $q/p$, i.e., 
\eqref{eq:dv} with $p$ and $q$ exchanged
\[\KL(q\,\|\,p)\;=\;\sup_f\;\E_q[f]-\log\E_p\big[e^{f}\big].\] 
The density $q$ now enters through a plain
expectation. 
Choosing $f=\A_q\vf$ makes this expectation vanish by \eqref{eq:stein}, and we obtain the lower bound 
\begin{equation}
  \KL(q\,\|\,p)\;\ge\;\Lv(\vf)\;\coloneqq\;-\log\E_p\big[e^{(\A_q\vf)(\x)}\big].
  \label{eq:stein-dv}
\end{equation}
The right-hand side depends on an expectation under $p$ and the score of $q$.  The inequality \eqref{eq:stein-dv} is tight when $-\A_q\vf$ equals $\log r$
up to a constant. 

Maximizing the right-hand side of \eqref{eq:stein-dv} over $\vf$ gives an estimate $\log r \approx -\A_q\hat{\vf} + C$, where $C$ is a constant.
We next define a finite-sample estimator and explain its statistical and dynamical interpretations.

\subsection{The regularized estimator}
\label{sec:method:estimator}

Let $\mathcal V$ be a class of vector fields with a norm $\|\cdot\|_{\mathcal V}$. We estimate the field by minimizing the empirical
version of $-\Lv(\vf)$ with a ridge penalty,
\begin{equation}
  \hat\vf\in\arg\min_{\vf\in\mathcal V}\;\; \log\frac{1}n\sum_{i=1}^n e^{(\A_q\vf)(\x_i)}+\frac{\lambda}{2}\|\vf\|_{\mathcal V}^2 .
  \label{eq:sdre}
\end{equation}
The ridge level $\lambda>0$ is a hyperparameter and is chosen using validation samples. Since
$(\A_q\vf)(\x_i)$ is linear in $\vf$, the objective is a log-sum-exp of linear functions of $\vf$ plus a convex squared-norm penalty. It is therefore
convex whenever $\mathcal V$ is a linear space, e.g., the span of finitely many fields or a reproducing kernel Hilbert space.
For the finite basis and quadratic coefficient penalty used in our experiments, we solve this convex problem with Newton or quasi-Newton methods
(Appendix~\ref{app:fields}). 

The lower bound \eqref{eq:stein-dv} at $\hat\vf$, approximated using validation samples $\x^{\mathrm{va}}_1,\dots,\x^{\mathrm{va}}_{n_{\mathrm{va}}}\sim p$
held out from \eqref{eq:sdre},
\begin{equation}
  \hat J\;\coloneqq\;-\log\frac{1}{n_{\mathrm{va}}}\sum_{i=1}^{n_{\mathrm{va}}} e^{(\A_q\hat\vf)(\x^{\mathrm{va}}_i)}\;\approx\;\Lv(\hat\vf)\;\le\;\KL(q\,\|\,p),
  \label{eq:jhat}
\end{equation}
can be used as a validation diagnostic:
a $\hat J$ well above its
standard error suggests that $q$ and $p$ differ. The empirical $\hat J$ is not itself a guaranteed lower bound.
Finally, the log ratio estimate is
\begin{equation}
  \log\hat r(\x)\;=\;-(\A_q\hat\vf)(\x)-\hat J,
  \label{eq:ratio-hat}
\end{equation}
where subtracting $\hat J$ normalizes the estimate so that $\frac{1}{n_{\mathrm{va}}}\sum_i1/\hat r(\x^{\mathrm{va}}_i)=1$.

Compared with Stein-parametrized linear ratio models used by
\citet{liu2019fisher}, in \ours{}, the Stein operator models a log ratio, not the ratio itself, 
and the field also defines a displacement of $q$, as we show next.

\subsection{One estimate, two perspectives on distribution shift}
\label{sec:method:sdf}

We now show how the field connects ratio estimation with sample transport. The first result identifies the log ratio when the population
bound is tight; the second describes how moving $q$ along the field changes its log-density.

\begin{proposition}
\label{prop:ratio}
Let $\KL(q\,\|\,p)<\infty$. Then $\Lv(\vf)\le\KL(q\,\|\,p)$, with equality if and only if $-(\A_q\vf)(\x)=\log r(\x)+\KL(q\,\|\,p)$ for $p$-almost
every $\x$.
\end{proposition}
When the population bound is attained, $-\A_q\vf$ is the log-ratio shifted by the divergence $\KL(q\,\|\,p)$. This also justifies subtracting $\hat J$ in \eqref{eq:ratio-hat}.
The proof is in Appendix~\ref{app:proofs}. 

The field $\vf$ also has a \emph{dynamical} reading. Move every point of $q$ along the field, $T_\eta(\x)=\x+\eta\vf(\x)$, where
$\eta>0$ is a step size, 
and write $(T_\eta)_\#q$ for the distribution of $T_\eta(\x)$ with $\x\sim q$.
The next result says how this
displacement changes $q$.

\begin{proposition}
\label{prop:displacement}
Let $\vf\in C^2$ with bounded first and second derivatives. For $\eta$ small enough, $T_\eta$ is a diffeomorphism and, for every $\x'$,
\begin{equation}
  \log\,(T_\eta)_\#q(\x')\;=\;\log q(\x')-\eta\,(\A_q\vf)(\x')+O(\eta^2).
  \label{eq:pushforward}
\end{equation}
\end{proposition}
See Appendix~\ref{app:proofs} for the proof. In words, moving $q$ by $\vf$ changes its log-density by $-\eta\,\A_q\vf$, the same quantity
that models $\log r$ in \eqref{eq:ratio-hat}. This change-of-variables argument also underlies the KL derivative used in SVGD
\citep{liu2016stein}.

Combining both propositions leads to an interesting interpretation of the optimal displacement $\vf$. 
To state this, recall that the Fisher--Rao gradient
flow of $\KL(\rho\,\|\,p)$ is $\partial_t\log\rho_t=-\big(\log\frac{\rho_t}{p}-\KL(\rho_t\,\|\,p)\big)$: instead of moving mass, it
reweights a density pointwise towards $p$ by the ratio $\frac{p}{\rho_t}$. The exact solution of this gradient flow is the geometric interpolation $\rho_t\propto\rho_0^{e^{-t}}p^{1-e^{-t}}$.

\begin{corollary}
\label{cor:fr}
Let $\vf$ satisfy the assumptions of Proposition~\ref{prop:displacement} and attain equality in Proposition~\ref{prop:ratio}. Then $\log(T_\eta)_\#q=\log q+\eta\log r+\eta\,\KL(q\,\|\,p)+O(\eta^2)$ pointwise for $p$-almost every $\x$.
This agrees with an explicit Euler step of the Fisher--Rao gradient flow started at $\rho_0=q$, up to $O(\eta^2)$.
The normalized geometric interpolation $q^{1-\eta}p^\eta$ has the same first-order log-density change.
\end{corollary}

The proof follows directly from Propositions~\ref{prop:ratio} and~\ref{prop:displacement}; we give it in Appendix~\ref{app:proofs}. 
Simply speaking, at the optimum, moving every sample from $q$ along the field $\vf$ moves the base $q$ closer to the target $p$ along the Fisher--Rao gradient flow. This interpretation distinguishes our displacement field $\vf$ from the ``feature function'' used by \citet{liu2019fisher}. The field $\vf$ is not only the function that best explains the density ratio $r$, but also a legitimate vector field whose displacement
multiplies $q$ by a small power of the ratio, thus moving $q$ towards $p$.

A single convex problem therefore provides two interpretations of the
shift: a statistical one (the ratio $\hat r$ and the diagnostic $\hat J$) and a dynamical one (the displacement $T_\eta$). Hence the name ``Stein displacement ratio estimation''. The Fisher--Rao flow is usually simulated by reweighting or by birth--death of particles
\citep{lu2019accelerating}, or by a kernel-based transport that requires pointwise evaluation of the unnormalized target-to-reference density ratio
\citep{maurais2024sampling}.
Each $T_\eta$ moves existing mass through an invertible map. Restricted fields and small steps can limit how much mass is moved between separated modes; we return to this limitation in Section~\ref{sec:exp:limits}.

Proposition~\ref{prop:displacement} also offers a principled way to incorporate dynamical prior knowledge into density ratio estimation: if we know that the distribution shift arises from a specific dynamical process, such as linear or low-rank transformations, we can restrict $\mathcal V$ to the corresponding family of fields.
Therefore, a structured class encodes prior knowledge of the distribution shift.
We use these restricted classes of transformations in our experiments, and they are explained in Appendix~\ref{app:fields}. 

Corollary~\ref{cor:fr} only describes a single Euler step $T_\eta$. In applications where we would like the base model to approximate the target $p$, we might want to perform such a step repeatedly. In the next section, we study two different algorithms that perform such tasks.

\section{Inference Algorithms}
\label{sec:alg}

The first-order Fisher--Rao interpretation motivates alternating ratio estimation and displacement to bring a starting base model closer to the target samples. 
At iteration $k$, we evaluate the score of the current
base, solve \eqref{eq:sdre} for a field $\hat\vf_k$, estimate $\hat r_k$ by \eqref{eq:ratio-hat}, and construct a candidate layer
$T_k=\mathrm{id}+\eta_k\hat\vf_k$. If the layer passes a validation criterion, we update the distributions and repeat; otherwise we
stop. 

We can carry out this estimation-and-transport loop from two different directions (Figure~\ref{fig:layer-construction}).
Either move the model towards the data, which we call
\emph{push-forward}, or pull the data back towards a simple, interpretable base distribution, which we call \emph{pull-back}. 

\begin{figure*}[t]
  \centering
  \resizebox{\textwidth}{!}{\begingroup
\colorlet{layerink}{black}
\colorlet{layerbg}{white}
\begin{tikzpicture}[
  font=\fontsize{9}{10.8}\selectfont,text=layerink,
  >={Stealth[length=3.5pt,width=3pt]},
  layer/.style={rectangle,rounded corners=1.5pt,minimum width=14mm,minimum height=5.3mm,
    inner sep=1.5pt,draw=layerink!24!layerbg,fill=layerink!4!layerbg,line width=.35pt},
  newlayer/.style={layer,draw=#1,fill=#1!10!layerbg,line width=.9pt},
  endpoint/.style={rectangle,rounded corners=2pt,minimum width=22mm,minimum height=6mm,
    inner sep=2pt,draw=layerink!30!layerbg,fill=layerink!7!layerbg,line width=.4pt},
  maparrow/.style={->,draw=layerink!65!layerbg,line width=.5pt},
  compare/.style={draw=layerink!45!layerbg,line width=.5pt,dash pattern=on 1pt off 2pt},
  small/.style={font=\fontsize{8}{9.6}\selectfont},
  secondary/.style={small,text=layerink!72!layerbg},
  grouping/.style={decorate,decoration={brace,amplitude=3pt},draw=layerink!60!layerbg,line width=.45pt}
]
\path[use as bounding box] (0,-.08) rectangle (17.145,4.10);

\node[anchor=west,font=\fontsize{9}{10.8}\selectfont\bfseries] at (.18,3.88) {(a) Push-forward};
\node[anchor=east] at (16.95,3.88) {$\Phi_k={\color{pushc}T_k}\circ\Phi_{k-1}$};
\node[anchor=west,font=\fontsize{9}{10.8}\selectfont\bfseries] at (.18,1.80) {(b) Pull-back};
\node[anchor=east] at (16.95,1.80) {$\Phi_k=\Phi_{k-1}\circ{\color{pullc}T_k}$};

\node[endpoint] (pprior) at (1.40,2.95) {Prior $\pi$};
\node[layer] (pbase) at (3.70,2.95) {$\Phi_0$};
\node[layer] (pfirst) at (5.25,2.95) {$T_1$};
\node[inner sep=0pt] (pdots) at (6.38,2.95) {$\cdots$};
\node[layer] (pprev) at (7.50,2.95) {$T_{k-1}$};
\node[newlayer=pushc] (pnew) at (9.15,2.95) {$T_k$};
\node[inner sep=2pt] (pout) at (11.05,2.95) {$q_{k+1}$};
\node[endpoint] (pdata) at (15.80,2.95) {Data $p$};
\draw[maparrow] (pprior.east) -- (pbase.west);
\draw[maparrow] (pbase.east) -- (pfirst.west);
\draw[maparrow] (pprev.east) -- (pnew.west);
\draw[maparrow] (pnew.east) -- (pout.west);
\draw[compare] (pout.east) -- (pdata.west);
\node[secondary,fill=layerbg,inner sep=2pt] at (13.00,2.95) {match to data};
\node[small] at (9.15,2.43) {new layer};
\draw[grouping] (8.20,2.53) -- (3.00,2.53)
  node[midway,below=4pt] {$\Phi_{k-1}$};
\draw[->,draw=pushc,line width=.9pt] (3.70,3.43) -- (9.85,3.43);
\node[small,fill=layerbg,inner sep=1.5pt] at (6.775,3.64) {Add layers};

\node[endpoint] (qprior) at (1.40,.87) {Prior $q_{\bm\beta}$};
\node[inner sep=2pt] (qout) at (6.15,.87) {$p_{k+1}$};
\node[newlayer=pullc] (qnew) at (8.05,.87) {$T_k^{-1}$};
\node[layer] (qprev) at (9.70,.87) {$T_{k-1}^{-1}$};
\node[inner sep=0pt] (qdots) at (10.825,.87) {$\cdots$};
\node[layer] (qfirst) at (11.95,.87) {$T_1^{-1}$};
\node[layer] (qbase) at (13.50,.87) {$\Phi_0^{-1}$};
\node[endpoint] (qdata) at (15.80,.87) {Data $p$};
\draw[maparrow] (qdata.west) -- (qbase.east);
\draw[maparrow] (qbase.west) -- (qfirst.east);
\draw[maparrow] (qprev.west) -- (qnew.east);
\draw[maparrow] (qnew.west) -- (qout.east);
\draw[compare] (qout.west) -- (qprior.east);
\node[secondary,fill=layerbg,inner sep=2pt] at (4.00,.87) {align / refit};
\node[small] at (8.05,.35) {new layer};
\draw[grouping] (14.20,.45) -- (9.00,.45)
  node[midway,below=4pt] {$\Phi_{k-1}^{-1}$};
\draw[->,draw=pullc,line width=.9pt] (13.50,1.35) -- (7.35,1.35);
\node[small,fill=layerbg,inner sep=1.5pt] at (10.425,1.56) {Add layers};
\end{tikzpicture}
\endgroup}
  \caption{
  Layer construction in the push-forward and pull-back algorithms.
  Push-forward appends $T_k$ to the model output; pull-back applies $T_k^{-1}$ to the current latent samples and refits the base between layers.
  In both panels, $\Phi_k$ maps the latent base to data.}
  \label{fig:layer-construction}
\end{figure*}
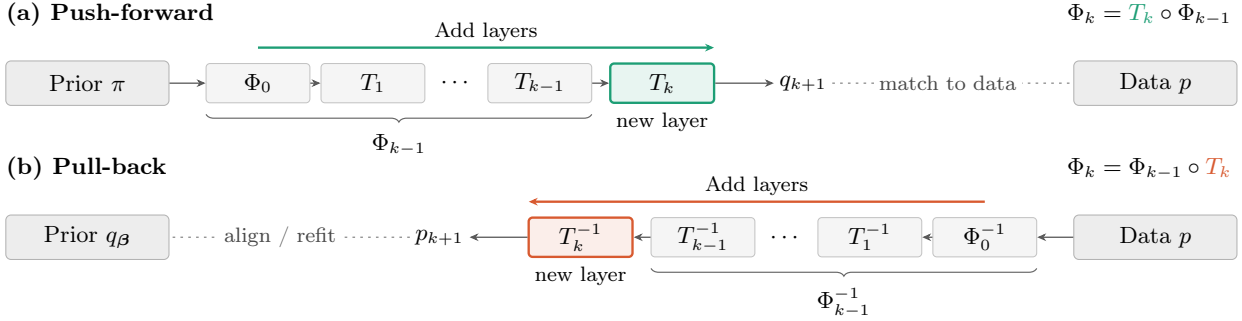

\subsection{Push-forward algorithm}
\label{sec:method:sampler}

Suppose our starting base model is a pretrained generative model and we want to adapt it so that it agrees with the target data. Let $\Phi_0$ be a pretrained invertible generator with a simple input distribution $\pi$, such as $\mathcal N(\bm0,I)$. We correct it by appending displacement layers to its output.
After $k-1$ layers, the current base distribution is
\begin{equation}
  q_k=(\Phi_{k-1})_\#\pi,\qquad \Phi_{k-1}=T_{k-1}\circ\dots\circ T_1\circ\Phi_0 ,
  \label{eq:push-model}
\end{equation}
with $q_1=(\Phi_0)_\#\pi$.
At iteration $k$,
we fit \ours{} using the new base $q_k$ and the target samples $\x_i$, then append the new layer at the
output: $\Phi_k=T_k\circ\Phi_{k-1}$. 

Under the conditions of Corollary~\ref{cor:fr}, with $\hat\vf_k$ attaining equality in Proposition~\ref{prop:ratio}, this iteration performs the update $q_{k+1}\propto q_k\hat r_k^{\,\eta_k}$ to first order in $\eta_k$ and the subsequent updates discretize the Fisher--Rao flow towards $p$. The pretrained generator remains fixed, and the accepted layers provide the correction.

At each iteration, the new estimate needs the score of the \emph{current} model $q_k$. When the sampler consists of invertible steps with tractable
Jacobians, we obtain $\log q_k$ by inverting those steps and accumulating their log-determinants, then compute its score by an
adjoint recursion (Appendix~\ref{app:backtrace}). This \emph{back-traced} score includes all corrections so far. 

We keep a layer only when the validation criterion detects a discrepancy between $p$ and $q_k$. We check for $\hat J_k>2\,\mathrm{se}(\hat J_k)$, where
$\mathrm{se}(\hat J_k)=\mathrm{sd}(1/\hat r_k(\x))/\sqrt{n_{\rm va}}$ is the delta-method estimate over the validation samples.
Finally, the procedure returns a corrected sampler whose density is available by change of variables. The pseudocode is given in Algorithm~\ref{alg:push}; Appendix~\ref{app:alg:push} gives further details. 

\begin{algorithm}[!htbp]
\caption{Push-forward, $\Phi\gets\textsc{Push}(\Phi_0,\{\x_i\})$}
\label{alg:push}
\begin{algorithmic}[1]
\Require pretrained invertible sampler $\Phi_0$ with latent distribution $\pi$; target samples $\{\x_i\}$, 
split into training and validation samples; 
maximum number of layers $K$
\For{$k=1,\dots,K$}
  \State Compute $\s_{q_k}(\x_i)$ by back-tracing $\Phi_{k-1}$
  \State $\hat\vf_k\gets$ solve \eqref{eq:sdre} with $q=q_k$ on training samples. 
  \State $T_k\gets\mathrm{id}+\eta_k\hat\vf_k$ \Comment{App.~\ref{app:alg:push}}
  \If{$\hat J_k\le 2\,\mathrm{se}(\hat J_k)$ on the validation samples} 
    \Return $\Phi_{k-1}$
  \EndIf
  \State $\Phi_k\gets T_k\circ\Phi_{k-1}$
\EndFor
\State \Return $\Phi_K$
\end{algorithmic}
\end{algorithm}

\subsection{Pull-back algorithm}
\label{sec:method:stack}

Suppose our goal is to find a latent representation of our data. The latent structure can be described by a simple, interpretable parametric prior $q$ whose log-density and score are known in closed form.

Before layer $k$, the current latent samples are $\bm z_i=\Phi_{k-1}^{-1}(\x_i)$, with distribution $p_k=(\Phi_{k-1}^{-1})_\#p$. We fit a field in this latent space and compose the resulting layer with the current map: $\Phi_k=\Phi_{k-1}\circ T_k$. The updated representation is therefore $\Phi_k^{-1}(\x_i)=T_k^{-1}(\bm z_i)$.

If the prior $q_{\bm\beta}$ has unknown parameters, we also need to align the prior and the current representation at every iteration. First, we fit $\bm\beta_k$ on the current latent samples $\bm z_i$. We then solve \eqref{eq:sdre} with base $q_{\bm\beta_k}$ and $\bm z_i$, and pull the data back
through the resulting layer: $\bm z_i\gets T_k^{-1}(\bm z_i)$ and update the forward map 
$\Phi_k=\Phi_{k-1}\circ T_k$. 
Under the conditions of Corollary~\ref{cor:fr}, with $\bm\beta_k$ held fixed, the forward move of the base is
$(T_k)_\#q_{\bm\beta_k}\propto q_{\bm\beta_k}\hat r_k^{\,\eta_k}$ to first order, where $\hat r_k$ estimates
$p_k/q_{\bm\beta_k}$.
Thus the same Fisher--Rao step is used in the latent space, interleaved with the prior update.
The score needed for \eqref{eq:sdre} is evaluated directly from $q_{\bm\beta_k}$, without back-tracing the full map.

We accept or reject layer $k$, chosen among several candidates (Appendix~\ref{app:alg:pull}), by the validation likelihood of the transformation model with $\bm\beta=\bm\beta_k$,
\begin{equation}
  \log\hat p(\x)=\log q_{\bm\beta}\big(\Phi_k^{-1}(\x)\big)+\log\big|\det\nabla\Phi_k^{-1}(\x)\big| .
  \label{eq:flow-lik}
\end{equation}
Algorithm~\ref{alg:stack} summarizes the procedure. Identifiability requires restricting the transformation: a sufficiently flexible $\Phi$ can match the prior without recovering the intended latent structure (Appendix~\ref{app:alg:pull}).

\begin{algorithm}[!htbp]
\caption{Pull-back, $(\Phi,\bm\beta)\gets\textsc{Pull}(\Phi_0,\{\x_i\})$}
\label{alg:stack}
\begin{algorithmic}[1]
\Require starting map $\Phi_0$; target samples $\{\x_i\}$, split into training and validation samples; base family $q_{\bm\beta}$;
maximum number of layers $K$
\State $\bm z_i\gets\Phi_0^{-1}(\x_i)$ for all target samples
\For{$k=1,\dots,K$}
  \State $\bm\beta_k\gets\arg\max_{\bm\beta}\sum_{\rm train}\log q_{\bm\beta}(\bm z_i)$
  \State Solve \eqref{eq:sdre} with $q=q_{\bm\beta_k}$ on training $\bm z_i$
  \State $T_k\gets\mathrm{id}+\eta_k\hat\vf_k$ \Comment{App.~\ref{app:alg:pull}}
  \If{$T_k$ does not increase the validation likelihood \eqref{eq:flow-lik}} \Return $(\Phi_{k-1},\bm\beta_k)$
  \EndIf
  \State $\bm z_i\gets T_k^{-1}(\bm z_i)$ for all latent samples
  \State $\Phi_k\gets\Phi_{k-1}\circ T_k$
\EndFor
\State $\bm\beta\gets\arg\max_{\bm\beta}\sum_{\rm train}\log q_{\bm\beta}(\bm z_i)$
\State \Return $(\Phi_K,\bm\beta)$
\end{algorithmic}
\end{algorithm}

\subsection{Choosing between push-forward and pull-back}
\label{sec:alg:choice}

The choice depends on the objective of inference. Push-forward corrects a model to match the data. Pull-back transforms the data to fit an interpretable prior. Their acceptance rules
reflect these goals: push-forward changes the model when the criterion detects a discrepancy, whereas pull-back keeps a layer
when it improves the validation likelihood.

Their computational costs also differ. Push-forward recomputes the score through the full sampler, including all previous layers.
Pull-back evaluates the prior score directly, carrying forward the latent samples and accumulated
log-determinants. In the next section, we use push-forward for guidance of a pretrained posterior sampler in simulation-based
inference, and pull-back for source recovery in nonlinear ICA.

\section{Applications}
\label{sec:app}

In this section, we examine the proposed algorithms (\textsc{Push, Pull}) on simulation-based inference and nonlinear ICA tasks. Specifically, we use push-forward to correct a posterior sampler and pull-back to recover independent sources, from scratch or by refining pretrained models. 

We use conditional variants of \textsc{Push} and \textsc{Pull}. We apply \eqref{eq:sdre} to joint target samples,
with the Stein operator acting only on the transported variable (Appendix~\ref{app:cond}).
In the experiments, \ours{} denotes our methods.

\subsection{Correcting a posterior sampler in simulation-based inference}
\label{sec:exp:sbi}

\begin{table*}[!t]\centering\small
\caption{Posterior correction in simulation-based inference: C2ST against reference posteriors
($0.5$ = indistinguishable; mean $\pm$ std). Best per column in bold, second best underlined (reference rows excluded).
$^\ast$No tractable density. $^\dagger$Hyperparameters selected by test C2ST.
$\times$: not applicable; --: not evaluated. The last row reports retraining on new target simulations as a reference.
Experimental details are in Appendix~\ref{app:sbi}.}
\label{tab:sbi}
\setlength{\tabcolsep}{3.6pt}
\begin{tabular}{l ccc cc cc}
\toprule
 & \multicolumn{3}{c}{Prior change} & \multicolumn{2}{c}{Simulator change} & \multicolumn{2}{c}{Partial coverage} \\
\cmidrule(lr){2-4}\cmidrule(lr){5-6}\cmidrule(lr){7-8}
Method & unimodal & bimodal & GLM & $n_{\rm real}{=}500$ & $n_{\rm real}{=}2000$ & covered & uncovered \\
\midrule
Frozen NPSE & 0.724{\scriptsize$\pm$.001} & 0.760{\scriptsize$\pm$.001} & 0.619{\scriptsize$\pm$.002} & 0.619{\scriptsize$\pm$.004} & 0.619{\scriptsize$\pm$.004} & 0.632{\scriptsize$\pm$.001} & 0.638{\scriptsize$\pm$.003} \\
\midrule
PriorGuide, ODE & 0.638{\scriptsize$\pm$.028} & 0.634{\scriptsize$\pm$.017} & 0.665{\scriptsize$\pm$.014} & $\times$ & $\times$ & $\times$ & $\times$ \\
PriorGuide + Langevin$^{\ast\dagger}$ & \textbf{0.511}{\scriptsize$\pm$.002} & \textbf{0.512}{\scriptsize$\pm$.001} & \textbf{0.543}{\scriptsize$\pm$.020} & $\times$ & $\times$ & $\times$ & $\times$ \\
Classifier DRE + SIR$^\ast$ & 0.816{\scriptsize$\pm$.005} & 0.862{\scriptsize$\pm$.003} & 0.717{\scriptsize$\pm$.005} & 0.722{\scriptsize$\pm$.009} & 0.721{\scriptsize$\pm$.007} & 0.727{\scriptsize$\pm$.004} & 0.759{\scriptsize$\pm$.007} \\
DG + Langevin$^\ast$ & 0.531{\scriptsize$\pm$.005} & \underline{0.555}{\scriptsize$\pm$.002} & 0.559{\scriptsize$\pm$.005} & \underline{0.589}{\scriptsize$\pm$.002} & \textbf{0.575}{\scriptsize$\pm$.004} & \underline{0.607}{\scriptsize$\pm$.009} & 0.859{\scriptsize$\pm$.036} \\
Fine-tuned NPSE & 0.582{\scriptsize$\pm$.007} & 0.629{\scriptsize$\pm$.024} & 0.564{\scriptsize$\pm$.006} & 0.620{\scriptsize$\pm$.004} & 0.596{\scriptsize$\pm$.003} & 0.615{\scriptsize$\pm$.009} & \underline{0.630}{\scriptsize$\pm$.008} \\
Fine-tuned, sim.\ + real & $\times$ & $\times$ & $\times$ & -- & -- & 0.612{\scriptsize$\pm$.007} & 0.633{\scriptsize$\pm$.002} \\
Frozen NPSE + \ours{} & \underline{0.523}{\scriptsize$\pm$.004} & 0.573{\scriptsize$\pm$.005} & \underline{0.552}{\scriptsize$\pm$.004} & \textbf{0.583}{\scriptsize$\pm$.006} & \underline{0.577}{\scriptsize$\pm$.003} & \textbf{0.590}{\scriptsize$\pm$.006} & \textbf{0.598}{\scriptsize$\pm$.007} \\
\midrule
Retrained on $10^5$ target sim. & 0.511{\scriptsize$\pm$.001} & 0.510{\scriptsize$\pm$.002} & 0.531{\scriptsize$\pm$.001} & 0.558{\scriptsize$\pm$.003} & 0.558{\scriptsize$\pm$.003} & 0.561{\scriptsize$\pm$.004} & 0.571{\scriptsize$\pm$.003} \\
\bottomrule
\end{tabular}
\end{table*}

In simulation-based inference, a parameter $\bm\theta$ is drawn from a prior $q(\bm\theta)$ and passed through a simulator
$q(\bm y\,|\,\bm\theta)$ to produce an observation $\bm y$. An amortized posterior sampler trained on these pairs can then generate samples for unseen 
observations. If the prior changes to $p(\bm\theta)$ or the simulator differs from the target process
$p(\bm y\,|\,\bm\theta)$, the pretrained model now targets the wrong posterior. We want to adapt our pretrained model $q(\bm\theta\,|\,\bm y)$ to $p(\bm\theta\,|\,\bm y)$ using target
pairs $(\bm\theta_i,\bm y_i) \sim p$ while keeping the pretrained network parameters fixed.

We consider the pretrained sampler $q$ to be a diffusion sampler \citep{song2021score,sharrock2024sequential}. Since diffusion sampling is a multi-level procedure, 
we apply $\textsc{Push}$ to correct its deterministic predictor steps at all noise levels, with a conditional field $\vf(\bm\theta,\bm y)$.  
At noise level $t$, the current sampler is compared with the target noised to the same level: the target parameters become
$a_t\bm\theta_i+\sigma_t\bm\xi_i$, where $a_t$ and $\sigma_t$ are the signal and noise scales and $\bm\xi_i\sim\mathcal N(\bm0,I)$; the observations are unchanged. The score is
back-traced through the predictor steps and all accepted corrections. The result is a sampler corrected for the target prior and simulator with a density available by change of variables at all steps.
Appendix~\ref{app:push-diffusion} gives the full procedure.

We use a neural posterior score estimator (NPSE) as the starting base model and adapt it on two \texttt{sbibm} tasks \citep{lueckmann2021benchmarking}: \texttt{gaussian\_linear} and
\texttt{bernoulli\_glm} (GLM), under three shifts (Table~\ref{tab:sbi}). \emph{Prior change}: the prior of \texttt{gaussian\_linear} is shifted or replaced
by a bimodal mixture; the GLM prior has a shifted bias and smoother filter. Target pairs are resampled from existing simulations
with weights proportional to $p(\bm\theta)/q(\bm\theta)$, without new simulations. \emph{Simulator change}: target pairs are generated
by a GLM with a steeper spiking nonlinearity. \emph{Partial coverage}: $10^3$ target pairs from this modified simulator cover only
negative-bias neurons; we evaluate on both covered and uncovered regions. We measure posterior accuracy by the classifier two-sample test (C2ST) against reference
samples ($0.5$ means indistinguishable). Retraining on $10^5$ new target simulations provides a reference.

We compare with PriorGuide, with and without Langevin refinement \citep{yang2025priorguide}, classifier-based DRE with
sampling--importance resampling (SIR) \citep{bickel2009discriminative}, discriminator guidance with Langevin refinement
\citep{kim2023refining}, and fine-tuning on target pairs, alone or mixed with simulations. PriorGuide applies only to prior changes;
methods marked $^\ast$ in Table~\ref{tab:sbi} have no tractable density. A tractable density can be useful in downstream SBI applications.  Appendix~\ref{app:sbi} gives the experimental setup,
baseline implementations and evaluation protocol.

For a unimodal prior change, the correction reduces C2ST from $0.724$ to $0.523$, close to retraining on new
target simulations ($0.511$), while using only the existing simulations. PriorGuide with Langevin refinement achieves lower mean C2ST in all three
prior-change settings, although its advantage over \ours{} is statistically significant only in the bimodal setting. PriorGuide with Langevin refinement uses
hyperparameters selected by test C2ST and does not retain a tractable density. Under simulator misspecification,
$n_{\rm real}\in\{500,2000\}$ target pairs improve the frozen sampler, with accuracy comparable to discriminator guidance and better than
fine-tuning. With partial coverage, the correction improves accuracy in both the covered and uncovered regions. Table~\ref{tab:sbi} gives the comparison; the experimental setup and statistical tests are in Appendix~\ref{app:sbi}.

\subsection{Recovering sources in nonlinear ICA}
\label{sec:exp:ica}

\begin{table*}[!t]\centering\small
\caption{Nonlinear ICA: validation MCC (mean $\pm$ std). Top: standalone estimators; bottom: pretrained models refined by \ours{}.
Best per column within each block is bold, excluding the starred standalone row. $^\ast$Uses oracle knowledge of the mixing architecture.
Experimental details are in Appendix~\ref{app:details}.}
\label{tab:main}
\begin{tabular}{l cc cc cc}
\toprule
 & \multicolumn{2}{c}{Synthetic ($d{=}5$)} & \multicolumn{2}{c}{CMU ARCTIC} & \multicolumn{2}{c}{LibriSpeech} \\
\cmidrule(lr){2-3}\cmidrule(lr){4-5}\cmidrule(lr){6-7}
Method & $n_c{=}200$ & $n_c{=}1000$ & $d{=}3$ & $d{=}5$ & $d{=}3$ & $d{=}5$ \\
\midrule
iVAE & 0.484{\scriptsize$\pm$0.045} & 0.492{\scriptsize$\pm$0.067} & 0.538{\scriptsize$\pm$0.090} & 0.383{\scriptsize$\pm$0.034} & 0.578{\scriptsize$\pm$0.131} & 0.461{\scriptsize$\pm$0.069} \\
GIN & 0.602{\scriptsize$\pm$0.087} & 0.635{\scriptsize$\pm$0.055} & \textbf{0.645}{\scriptsize$\pm$0.084} & 0.440{\scriptsize$\pm$0.037} & 0.575{\scriptsize$\pm$0.077} & 0.438{\scriptsize$\pm$0.037} \\
\ours{} (from scratch) & \textbf{0.689}{\scriptsize$\pm$0.061} & \textbf{0.673}{\scriptsize$\pm$0.062} & 0.621{\scriptsize$\pm$0.150} & \textbf{0.518}{\scriptsize$\pm$0.047} & \textbf{0.770}{\scriptsize$\pm$0.068} & \textbf{0.537}{\scriptsize$\pm$0.026} \\
OSF$^\ast$ & 0.795{\scriptsize$\pm$0.068} & 0.762{\scriptsize$\pm$0.084} & 0.854{\scriptsize$\pm$0.024} & 0.747{\scriptsize$\pm$0.046} & 0.850{\scriptsize$\pm$0.014} & 0.736{\scriptsize$\pm$0.046} \\
\midrule
GIN + \ours{} & 0.660{\scriptsize$\pm$0.104} & 0.687{\scriptsize$\pm$0.066} & 0.747{\scriptsize$\pm$0.090} & 0.587{\scriptsize$\pm$0.081} & 0.588{\scriptsize$\pm$0.056} & 0.467{\scriptsize$\pm$0.043} \\
OSF$^\ast$ + \ours{} & \textbf{0.866}{\scriptsize$\pm$0.068} & \textbf{0.887}{\scriptsize$\pm$0.049} & \textbf{0.883}{\scriptsize$\pm$0.007} & \textbf{0.805}{\scriptsize$\pm$0.031} & \textbf{0.907}{\scriptsize$\pm$0.010} & \textbf{0.811}{\scriptsize$\pm$0.019} \\
\bottomrule
\end{tabular}
\end{table*}

\begin{table*}[!t]\centering\small
\caption{Refinement versus extended training: validation MCC (seed means). Extended training uses $4\times$ the original training steps;
``wins'' counts seeds on which refinement beats extended training. Best per column within each model is bold.
Experimental details are in Appendix~\ref{app:details}.}
\label{tab:control}
\setlength{\tabcolsep}{5.8pt}
\begin{tabular}{l cccccc c}
\toprule
 & \multicolumn{2}{c}{Synthetic} & \multicolumn{2}{c}{ARCTIC} & \multicolumn{2}{c}{LibriSpeech} & \\
 & $n_c{=}200$ & $n_c{=}1000$ & $d{=}3$ & $d{=}5$ & $d{=}3$ & $d{=}5$ & wins \\
\midrule
GIN & 0.602 & 0.635 & 0.645 & 0.440 & 0.575 & 0.438 &  \\
\quad Extended training & 0.619 & 0.658 & 0.623 & 0.433 & 0.587 & 0.445 &  \\
\quad + \ours{} & \textbf{0.660} & \textbf{0.687} & \textbf{0.747} & \textbf{0.587} & \textbf{0.588} & \textbf{0.467} & 19/22 \\
\midrule
OSF & 0.795 & 0.762 & 0.854 & 0.747 & 0.850 & 0.736 &  \\
\quad Extended training & 0.837 & 0.856 & 0.858 & 0.771 & 0.833 & 0.790 &  \\
\quad + \ours{} & \textbf{0.866} & \textbf{0.887} & \textbf{0.883} & \textbf{0.805} & \textbf{0.907} & \textbf{0.811} & 20/22 \\
\bottomrule
\end{tabular}
\end{table*}

In independent component analysis (ICA), 
observations are  $\x=\bm f(\bm z)$ for an invertible mixing $\bm f$ and latent sources $\bm z=(z_1,\dots,z_d)$. We want to recover $\bm z$ from observations $\bm x$.
In auxiliary variable ICA, we assume that the sources are independent given a time segment $\bm c$,
$p(\bm z\,|\,\bm c)=\prod_j p_j(z_j\,|\,\bm c)$. Under suitable regularity conditions across segments, the inverse mixing is
identifiable up to a permutation and componentwise invertible transformations
\citep{hyvarinen2016unsupervised,hyvarinen2019nonlinear,khemakhem2020variational}. The goal is to recover these sources from the
observations and segment labels.

Since the latent signals are conditionally independent, we use a factorized conditional base
$q_{\bm\beta}(\bm z\,|\,\bm c)$. Its parameters vary across segments, while the field and hence the learned transformation are shared across segments.
The sources are read off as $\Phi^{-1}(\x)$. We use Gaussian or smooth super-Gaussian
bases and restrict the fields to structured classes that are linear in their parameters (Appendix~\ref{app:fields}). Each candidate field is therefore fitted by a convex
problem, and the restrictions express assumptions about the mixing. We also test two different starting configurations: starting from
$\Phi_0=\mathrm{id}$ learns the model \textbf{from scratch}; starting from a pretrained model performs \textbf{refinement} by appending
layers to its existing transformation.

Our experiments use synthetic Gaussian sources and real speech recordings from CMU ARCTIC \citep{kominek2004cmu} and LibriSpeech \citep{panayotov2015librispeech}, all mixed by a random
invertible nonlinear network (Appendix~\ref{app:details}). The two synthetic settings use $n_c\in\{200,1000\}$ samples per segment.
We measure recovery by the mean correlation coefficient (MCC) between estimated and true sources
after optimal matching. 

We compare with three estimators trained end to end by maximum likelihood or its bound. iVAE \citep{khemakhem2020variational} is a
variational autoencoder whose prior on the latent variables is conditioned on the segment label; its exact likelihood is intractable. GIN
\citep{sorrenson2020disentanglement} is a volume-preserving normalizing flow with a conditional Gaussian base. The oracle-structured flow (OSF) uses oracle knowledge of the mixing structure,
which is not available in practice (Appendix~\ref{app:details}). We refine GIN and OSF, both with tractable exact likelihoods.

As Table~\ref{tab:main} shows, refinement with \ours{} improves the mean MCC of every pretrained model in all six
settings. For example, OSF improves from $0.762$ to $0.887$ on the synthetic benchmark ($n_c=1000$). To check whether this gain merely reflects additional training, we also train each model for four times the original number of steps (extended training).
Table~\ref{tab:control} shows that refinement with \ours{} gives higher MCC in 39 of 44 paired runs. These comparisons support using
pull-back to improve pretrained representations.

Trained from scratch, \ours{} achieves the highest mean MCC among the evaluated estimators without oracle knowledge in five of the six
settings. OSF, which has access to the mixing architecture, achieves higher mean MCC in all six settings. Appendix~\ref{app:details} gives
the data, models and detailed comparisons; Appendix~\ref{app:results} examines stopping rules, selection among refined models and further
datasets.

\section{Conclusion and limitations}
\label{sec:conclusion}
\label{sec:exp:limits}

We propose \ours{}, which models the log density ratio as minus the Stein operator applied to a displacement field, up to a normalizing constant. This connects two views of distribution shift. A convex fit estimates both the ratio and a KL lower bound. When the population bound is attained, the estimated field gives a first-order Fisher--Rao update. Repeating the fit-and-update step yields push-forward and pull-back inference algorithms, which we evaluate on posterior correction in SBI and source recovery in nonlinear ICA.

Our method has theoretical and computational limitations. Moving mass between separated modes may require complex displacement fields. Restricting $\mathcal V$ can limit such transport, whereas highly flexible, unregularized fields can make ratio estimation numerically unstable in high dimensions. Push-forward also requires back-tracing the score through all previous corrections, so its cost grows with the number of layers.

Due to space constraints, we defer a detailed discussion of related work to Section~\ref{sec:related}.
There, we compare SDRE with existing density ratio estimators, discuss its transport-based approximation of Fisher--Rao dynamics in relation to kernel-based approaches, and contrast greedy, layer-wise pull-back estimation with end-to-end nonlinear ICA methods such as GIN \citep{sorrenson2020disentanglement}.

\clearpage

\subsection*{AI Use Statement}
In this work, we used generative AI tools for writing and running experiment code, 
including the implementation of \textsc{Push} and \textsc{Pull} and the baselines, for orchestrating experiments on HPC, and for drafting and proofreading the manuscript, structuring tables and formatting the bibliography. The tools also assisted with mathematical derivations and the writing and checking of proofs. They also took part in discussions of experimental design and the interpretation of experimental results, and some experimental design choices (the rank-$d$ field class, the refinement protocol and the multi-start selection) were first suggested in such discussions and then tested by the authors.
We have reviewed all AI-assisted work:
every reported number was checked against the experiment logs, and the code, proofs and references were reviewed by the authors. 
We take responsibility for the 
final content of this work, including text, claims, or artifacts produced with the aid of generative AI.

\bibliography{references}

\clearpage
\appendix
\thispagestyle{empty}
\onecolumn
\aistatstitle{\papertitle: \\ Supplementary Materials}
\section{Proofs}
\label{app:proofs}

\begin{proof}[Proof of Proposition~\ref{prop:ratio}]
Write $f=\A_q\vf$ and $Z=\E_p[e^f]$. If $Z=\infty$, then $\Lv(\vf)=-\infty$ and the inequality is immediate.
Otherwise, $0<Z<\infty$ and the tilted density $p_{\vf}=p\,e^f/Z$ is well defined. Stein's identity gives
\begin{align*}
  \KL(q\,\|\,p_{\vf})
  &=\E_q\!\left[\log\frac{q}{p}-f+\log Z\right]\\
  &=\KL(q\,\|\,p)+\log Z
   =\KL(q\,\|\,p)-\Lv(\vf)\ge0.
\end{align*}
Equality holds exactly when $p_{\vf}=q$ almost everywhere, or equivalently
$f=\log(q/p)+\log Z$ for $p$-almost every $\x$.
Taking the $q$-expectation yields $0=\KL(q\,\|\,p)+\log Z$.
Thus equality is equivalent to $-\A_q\vf=\log r+\KL(q\,\|\,p)$, as claimed.
\end{proof}

\begin{proof}[Proof of Proposition~\ref{prop:displacement}]
Let $L=\sup_{\x}\|\nabla\vf(\x)\|_{\mathrm{op}}<\infty$ and take $\eta>0$ with $\eta L<1$.
For each $\x'$, the map $\x\mapsto\x'-\eta\vf(\x)$ is a contraction on $\R^d$, so it has a unique fixed point $T_\eta^{-1}(\x')$.
The Jacobian $I+\eta\nabla\vf$ is nonsingular and has positive determinant, since it varies continuously from $I$ without becoming singular.
The inverse function theorem therefore makes $T_\eta$ a diffeomorphism.

Fix $\x'$ and put $\x=T_\eta^{-1}(\x')$. The inverse equation and the Lipschitz bound give
$\|\x-\x'\|\le\eta\|\vf(\x')\|/(1-\eta L)=O(\eta)$, and hence
\[
  \x=\x'-\eta\vf(\x)=\x'-\eta\vf(\x')+O(\eta^2).
\]
Expanding the log-density and log-determinant separately gives
\begin{align*}
  \log q(\x)
   &=\log q(\x')-\eta\,\s_q(\x')^\T\vf(\x')+O(\eta^2),\\
  \log\det(I+\eta\nabla\vf(\x))
   &=\eta\,\nabla\!\cdot\vf(\x')+O(\eta^2).
\end{align*}
Subtracting these expressions in the change-of-variables formula yields
\[
  \log(T_\eta)_\#q(\x')
  =\log q(\x')-\eta\,(\A_q\vf)(\x')+O(\eta^2).
\]
All remainders are pointwise in $\x'$; their constants may depend on $\x'$.
\end{proof}

\begin{proof}[Proof of Corollary~\ref{cor:fr}]
Combining Propositions~\ref{prop:ratio} and~\ref{prop:displacement} gives
\[
  \log(T_\eta)_\#q
  =\log q+\eta\big[\log r+\KL(q\,\|\,p)\big]+O(\eta^2).
\]
At $\rho_0=q$, the Fisher--Rao log-density velocity is
\[
  \left.\partial_t\log\rho_t\right|_{t=0}
  =-\log(q/p)+\KL(q\,\|\,p)
  =\log r+\KL(q\,\|\,p).
\]
Thus the displacement agrees with an explicit Euler step for this velocity up to $O(\eta^2)$.

The geometric interpolation has the same first-order change after normalization.
To see this, write $\widetilde q_\eta=q\,r^\eta/Z_\eta$, where $Z_\eta=\E_q[r^\eta]$ and $Z_0=1$.
Multiplying by $r^\eta$ adds $\eta\log r$ to the log-density; we only need to account for the normalizing constant.
Its change at $\eta=0$ is
\[
  \left.\frac{d}{d\eta}\log Z_\eta\right|_{\eta=0^+}
  =\E_q[\log r]=-\KL(q\,\|\,p).
\]
This derivative is well defined under the finite-KL assumption: $\E_q|\log r|<\infty$, since $\log r\le r$ when $r\ge1$.
Also, for $0\le\eta\le1/2$, $|r^\eta\log r|$ is bounded by a constant times $r+|\log r|$, so we may differentiate inside the expectation.

Thus $\log Z_\eta=-\eta\,\KL(q\,\|\,p)+o(\eta)$, and normalizing gives
\[
  \log\widetilde q_\eta
  =\log q+\eta\big[\log r+\KL(q\,\|\,p)\big]+o(\eta).
\]
The coefficient of $\eta$ is exactly the Fisher--Rao velocity above.
\end{proof}

\section{Conditional distributions}
\label{app:cond}

Both applications in Section~\ref{sec:app} are conditional. Let $\bm c$ be a conditioning variable with distribution $p(\bm c)$, and let
$q(\cdot\,|\,\bm c)$ and $p(\cdot\,|\,\bm c)$ be a base and a target distribution for every $\bm c$. The Stein operator acts on the
transported variable for each fixed $\bm c$,
\begin{equation*}
  (\A_q\vf)(\x,\bm c)=\s_q(\x\,|\,\bm c)^\T\vf(\x,\bm c)+\nabla_{\x}\!\cdot\vf(\x,\bm c),\qquad \s_q(\x\,|\,\bm c)=\nabla_{\x}\log q(\x\,|\,\bm c),
\end{equation*}
and Stein's identity $\E_{q(\cdot|\bm c)}[(\A_q\vf)(\x,\bm c)]=0$ holds for every $\bm c$. Given joint target samples
$(\x_i,\bm c_i)\sim p(\bm c)\,p(\x\,|\,\bm c)$, the estimator \eqref{eq:sdre} is applied with $(\A_q\vf)(\x_i,\bm c_i)$ in place of
$(\A_q\vf)(\x_i)$, and the population Stein--DV objective lower-bounds the expected divergence,
\begin{equation}
  -\log\E_{p(\x,\bm c)}\big[e^{(\A_q\vf)(\x,\bm c)}\big]\le\E_{p(\bm c)}\KL\big(q(\cdot\,|\,\bm c)\,\|\,p(\cdot\,|\,\bm c)\big).
  \label{eq:cond-bound}
\end{equation}
Indeed, \eqref{eq:stein-dv} applied for each $\bm c$ gives $\E_{p(\x|\bm c)}[e^{(\A_q\vf)(\x,\bm c)}]\ge e^{-\KL_{\bm c}}$ with
$\KL_{\bm c}=\KL(q(\cdot|\bm c)\|p(\cdot|\bm c))$, and averaging over $p(\bm c)$ and applying Jensen's inequality,
$\log\E_{p(\bm c)}e^{-\KL_{\bm c}}\ge-\E_{p(\bm c)}\KL_{\bm c}$, gives \eqref{eq:cond-bound}. A single target sample per conditioning value
therefore suffices to form the empirical objective, and no samples from the base are needed.
Under the assumptions of Proposition~\ref{prop:displacement}, its expansion holds at each fixed $\bm c$.
The Fisher--Rao conclusion of Corollary~\ref{cor:fr} additionally requires equality in the conditional DV bound at that $\bm c$.
A restricted field shared across conditions need not attain these equalities.
Since a single $\hat J$ is shared across conditions, \eqref{eq:ratio-hat} provides a pooled normalization; condition-specific constants remain. In simulation-based inference the transported variable is the parameter $\bm\theta$, the conditioning variable is
the observation $\bm y$, and the field depends on both; in nonlinear ICA only the base depends on the segment label $\bm c$, and the
field is shared across segments.

\section{Details of the inference algorithms}
\label{app:alg}

This appendix collects the details of the two algorithms of Section~\ref{sec:alg}. Each layer is $T_k=\mathrm{id}+\eta_k\hat\vf_k$,
where $\hat\vf_k$ solves \eqref{eq:sdre} for the current pair of distributions.
When the population bound is attained, Corollary~\ref{cor:fr} identifies the layer with a first-order Fisher--Rao update.
The step size must also keep $T_k$ invertible, which
$\eta_k\sup\|\nabla\hat\vf_k\|<1$ guarantees.

Figure~\ref{fig:algos} summarizes the iteration loops.

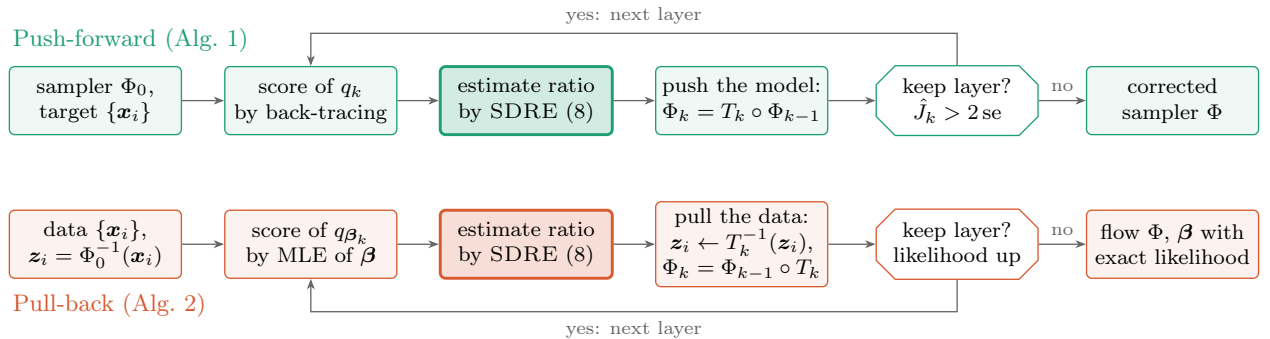
\begin{figure}[!htbp]
  \centering
  \begin{tikzpicture}[
  font=\fontsize{8}{9.6}\selectfont, >={Stealth[length=4pt]},
  box/.style={draw=#1, fill=#1!7, rounded corners=2pt, align=flush center, inner sep=2.5pt, text width=21mm,
              minimum height=9mm, line width=0.5pt},
  est/.style={box=#1, fill=#1!20, line width=1.1pt},
  test/.style={draw=#1, fill=white, chamfered rectangle, chamfered rectangle xsep=1.5mm, align=flush center, inner sep=2pt,
               text width=17mm, minimum height=9mm, line width=0.5pt},
  rowlab/.style={text=#1, font=\fontsize{9}{11}\selectfont},
  flow/.style={->, draw=black!60, line width=0.5pt},
  yn/.style={font=\fontsize{7}{8.4}\selectfont, text=black!60}]
\def\yP{0} \def\yQ{-1.9}
\def\Xa{1.15} \def\Xb{4.0} \def\Xc{6.85} \def\Xd{9.7} \def\Xe{12.55} \def\Xf{15.4}
\node[box=pushc]  (P1) at (\Xa, \yP) {sampler $\Phi_0$,\\ target $\{\x_i\}$};
\node[rowlab=pushc, anchor=west] at ([xshift=-2pt]P1.west |- 0, \yP+0.8) {Push-forward (Alg.~\ref{alg:push})};
\node[box=pushc]  (P2) at (\Xb, \yP) {score of $q_k$\\ by back-tracing};
\node[est=pushc]  (P3) at (\Xc, \yP) {estimate ratio\\ by \ours{} \eqref{eq:sdre}};
\node[box=pushc]  (P4) at (\Xd, \yP) {push the model:\\ $\Phi_k=T_k\circ\Phi_{k-1}$};
\node[test=pushc] (P5) at (\Xe, \yP) {keep layer?\\ $\hat J_k>2\,\mathrm{se}$};
\node[box=pushc]  (P6) at (\Xf, \yP) {corrected\\ sampler $\Phi$};
\node[box=pullc]  (Q1) at (\Xa, \yQ) {data $\{\x_i\}$,\\ $\bm z_i=\Phi_0^{-1}(\x_i)$};
\node[rowlab=pullc, anchor=west] at ([xshift=-2pt]Q1.west |- 0, \yQ-0.8) {Pull-back (Alg.~\ref{alg:stack})};
\node[box=pullc]  (Q2) at (\Xb, \yQ) {score of $q_{\bm\beta_k}$\\ by MLE of $\bm\beta$};
\node[est=pullc]  (Q3) at (\Xc, \yQ) {estimate ratio\\ by \ours{} \eqref{eq:sdre}};
\node[box=pullc]  (Q4) at (\Xd, \yQ) {pull the data:\\ $\bm z_i\leftarrow T_k^{-1}(\bm z_i)$,\\ $\Phi_k=\Phi_{k-1}\circ T_k$};
\node[test=pullc] (Q5) at (\Xe, \yQ) {keep layer?\\ likelihood up};
\node[box=pullc]  (Q6) at (\Xf, \yQ) {flow $\Phi$, $\bm\beta$ with\\ exact likelihood};
\foreach \r in {P, Q} {
  \foreach \a/\b in {1/2, 2/3, 3/4, 4/5} \draw[flow] (\r\a) -- (\r\b);
  \draw[flow] (\r5) -- node[yn, above] {no} (\r6);
}
\draw[flow] (P5.north) -- ++(0, 0.42) -| node[yn, above, pos=0.25] {yes: next layer} (P2.north);
\draw[flow] (Q5.south) -- ++(0, -0.42) -| node[yn, below, pos=0.25] {yes: next layer} (Q2.south);
\end{tikzpicture}
  \caption{Iteration loops for the two inference algorithms of Section~\ref{sec:alg}.
  }
  \label{fig:algos}
\end{figure}

\subsection{Push-forward}
\label{app:alg:push}

\paragraph{Step size and acceptance.} The step size is $\eta_k=\min\{1,1/(2\max_i\|\nabla\hat\vf_k(\x_i)\|)\}$, with the maximum over the
target samples.
This sample-based rule does not by itself guarantee global invertibility. The ridge level $\lambda$ is chosen from a set $\Lambda$ by the
validation value of $\hat J_k$, and $\hat J_k$ and $\mathrm{se}(\hat J_k)$ are then computed on the same validation samples
(Appendix~\ref{app:sbi} gives the values used in our experiments). Because the validation samples also choose $\lambda$, the test
$\hat J_k>2\,\mathrm{se}(\hat J_k)$ is somewhat optimistic. %

\subsection{Pull-back}
\label{app:alg:pull}

\paragraph{Candidate layers.} Every layer is selected by likelihood. For each candidate field class in a set $\mathcal F$
(Appendix~\ref{app:fields}), ridge level in a set $\Lambda$ and step size in a set $H$, we pull the validation samples back through the
candidate layer and evaluate their likelihood \eqref{eq:flow-lik} with $\bm\beta=\bm\beta_k$ (Appendix~\ref{app:pull}). Nonlinear
candidates whose Jacobian determinant falls below a small threshold at a training sample are discarded.
This checks local nonsingularity at the sampled points; it does not establish global injectivity.
A linear layer is globally invertible whenever $\det(I+\eta A)\ne0$. We keep the candidate with the highest validation likelihood,
and we stop when no candidate increases it. After the last layer, $\bm\beta$ is refitted by maximum likelihood on the final $\bm z_i$.
The sets $\mathcal F$, $\Lambda$ and $H$ used for nonlinear ICA are listed in Appendix~\ref{app:details}.

\paragraph{Restricted transports.} Pull-back fits the transformation model by maximum likelihood in a greedy way: every layer is proposed
by \ours{} and selected by validation likelihood. However, this is only informative for a \emph{restricted} class of transports. For an unconditional base, a sufficiently flexible map can compensate for changes in $\bm\beta$, producing the same data density from different latent representations. The data likelihood alone then cannot identify the base parameters or the sources. Conditional structure and restrictions on the shared map can resolve this ambiguity under additional identifiability assumptions. The field classes in $\mathcal F$, the number of layers and the selection by validation likelihood restrict the transport, and they
form part of our model assumption. Push-forward does not have this problem, because the target is to fit a model to the target data points. 

\paragraph{Relation to known constructions.} First, the model is assembled from small ratio estimates between consecutive intermediate
distributions, in the spirit of telescoping density ratio estimation \citep{rhodes2020telescoping}, but the bridges are produced by
transport rather than fixed in advance. Second, at the population level, the first-order term in $\vf$ of the objective corresponding to \eqref{eq:sdre} at layer $k$ is
$\E_{p_k}[\A_{q_{\bm\beta_k}}\vf]$. With the base held at $q_{\bm\beta_k}$, composing the layer $T=\mathrm{id}+\eta\vf$ changes the average
log-likelihood \eqref{eq:flow-lik} of the data by $-\eta\,\E_{p_k}[\A_{q_{\bm\beta_k}}\vf]$ to first order (the score term comes from
$q_{\bm\beta_k}$ and the divergence term from the log-determinant), so minimizing the objective follows the ascent direction of the
likelihood-based normalizing flows of \citet{tabak2010density}. The log-moment  objective also captures higher-order variation beyond this linear term. We also fit a field using a restricted family and use \eqref{eq:jhat} as an empirical estimate of the population lower bound on the remaining divergence, rather than a guaranteed finite-sample bound.

\section{Back-traced score}
\label{app:backtrace}

Let the sampler be a composition of invertible residual maps, $\Phi=G_N\circ\dots\circ G_1$ with $G_j=\mathrm{id}+\bm h_j$, applied to a
latent variable $\bm z$ with latent distribution $\pi$, and let $q=\Phi_\#\pi$. The maps $G_j$ are the predictor steps of the
pretrained sampler and the layers appended so far. Given $\x$, invert the maps one at a time, $\x^{(N)}=\x$ and
$\x^{(j-1)}=G_j^{-1}(\x^{(j)})$, so that $\bm z=\x^{(0)}$; then
\begin{equation}
  \log q(\x)=\log\pi(\x^{(0)})-\sum_{j=1}^N\log\big|\det\big(I+\nabla\bm h_j(\x^{(j-1)})\big)\big|.
  \label{eq:backtrace}
\end{equation}
Differentiating \eqref{eq:backtrace} through the inverses gives the adjoint recursion
\begin{equation*}
  \bm\zeta_0=\nabla\log\pi(\x^{(0)}),\qquad
  \bm\zeta_j=\big(I+\nabla\bm h_j(\x^{(j-1)})\big)^{-\T}\Big[\bm\zeta_{j-1}-\nabla\log\big|\det(I+\nabla\bm h_j)\big|(\x^{(j-1)})\Big],\qquad
  \s_q(\x)=\bm\zeta_N,
\end{equation*}
with $\bm\zeta_0=-\x^{(0)}$ for $\pi=\mathcal N(\bm0,I)$; the recursion keeps one Jacobian in memory at a time. Each $G_j^{-1}$ is computed by
a damped Newton iteration. 
Target samples at which the inversion does not converge (residual of the recovered $\x^{(0)}$ above
$10^{-7}$) are treated as outside the sampler's image and are dropped before the field is fitted.

\section{Pull-back: likelihood and inverses}
\label{app:pull}

At step $k$ of Algorithm~\ref{alg:stack}, each sample carries its current latent value $\bm z_i=\Phi_{k-1}^{-1}(\x_i)$ and the accumulated
log-determinant $D_i=\log|\det\nabla\Phi_{k-1}^{-1}(\x_i)|$, initialized at $\Phi_0$ and updated after an accepted layer as
$D_i\gets D_i-\log|\det\nabla T_k(\bm z_i^{\rm new})|$ with $\bm z_i^{\rm new}=T_k^{-1}(\bm z_i)$. For a candidate layer
$T=\mathrm{id}+\eta\hat\vf$, the validation likelihood after the pull-back is
\begin{equation*}
  \frac{1}{n_{\rm va}}\sum_{i\in\rm va}\Big[\log q_{\bm\beta_k}\big(T^{-1}(\bm z_i)\big)-\log\big|\det\nabla T\big(T^{-1}(\bm z_i)\big)\big|+D_i\Big],
\end{equation*}
and adding no layer corresponds to $T=\mathrm{id}$. A nonlinear candidate is discarded if $\nabla T$ is close to singular at a training sample (determinant or, for the coordinatewise
class, a diagonal entry below a threshold $\delta$; Appendix~\ref{app:details}); a linear candidate is invertible whenever
$\det(I+\eta A)\ne0$. The inverse $T^{-1}(\bm z)$ is computed by the fixed-point iteration $\bm z'\gets\bm z-\eta\hat\vf(\bm z')$, which is
one-dimensional for the rank-one class of Appendix~\ref{app:fields}, and by Newton's method for the rank-$d$ class.

\section{Field classes}
\label{app:fields}

\paragraph{Finite bases.} With $m$ basis fields collected in $\bm\Psi(\x)\in\R^{d\times m}$, a field is $\vf(\x)=\bm\Psi(\x)\bm\alpha$
with coefficients $\bm\alpha\in\R^m$, and the Stein operator is linear in $\bm\alpha$,
\begin{equation}
  (\A_q\vf)(\x)=\bm\phi(\x)^\T\bm\alpha,\qquad \bm\phi(\x)=\bm\Psi(\x)^\T\s_q(\x)+\nabla\!\cdot\bm\Psi(\x),
  \label{eq:stein-features}
\end{equation}
where $\nabla\!\cdot\bm\Psi$ is the column-wise divergence. We call $\bm\phi$ the \emph{Stein features}, following the tradition of \citet{liu2019fisher}. With $\|\vf\|_{\mathcal V}=\|\bm\alpha\|$,
\eqref{eq:sdre} becomes $\min_{\bm\alpha}\log\frac{1}n\sum_ie^{\bm\phi(\x_i)^\T\bm\alpha}+\frac{\lambda}{2}\|\bm\alpha\|^2$, a log-sum-exp of
linear functions plus a quadratic, which is strictly convex and solved by L-BFGS or damped Newton. It is the log-linear KLIEP problem
\eqref{eq:kliep} with the roles of $p$ and $q$ exchanged, in which the linear term, an average over samples of $q$, is replaced by its
expectation, which is zero by Stein's identity.

\paragraph{Classes for pull-back.} With a Gaussian radial basis function $\mathrm k(t,\gamma)=\exp\big(-(t-\gamma)^2/(2\omega^2)\big)$ with
bandwidth $\omega$ on $L$ fixed centers $\gamma_1,\dots,\gamma_L$, we use five classes, each linear in its parameters. In pull-back they
act on the latent variable $\bm z=(z_1,\dots,z_d)$; we index its coordinates, and the random directions, by $l=1,\dots,d$:
\begin{itemize}
  \item \emph{linear} ($d^2$ parameters): $\vf(\bm z)=A\bm z$, with Stein features $\bm\phi(\bm z)=\mathrm{vec}\big(\s_q(\bm z)\bm z^\T+I\big)$;
  \item \emph{coordinatewise} ($dL$): $v_l(\bm z)=\sum_j\alpha_{lj}\mathrm k(z_l,\gamma_j)$ with $\alpha_{lj}\in\R$; diagonal Jacobian,
  reshapes the marginals;
  \item \emph{joint} ($d^2+dL$): linear plus coordinatewise, with separate parameters $A$ and $\alpha_{lj}$;
  \item \emph{rank-one} ($dL$): $\vf(\bm z)=\sum_j\bm\alpha_j\mathrm k(\bm u^\T\bm z,\gamma_j)$, $\bm\alpha_j\in\R^d$, a planar-flow-like layer
  \citep{rezende2015variational} with $\det\nabla T=1+\eta\,\bm u^\T\sum_j\bm\alpha_j\mathrm k'(\bm u^\T\bm z,\gamma_j)$, where $\mathrm k'$ is the
  derivative in the first argument; the problem is convex for a fixed direction $\bm u$, which is chosen among random candidates by
  the value of \eqref{eq:sdre};
  \item \emph{rank-$d$} ($d^2L$): $\vf(\bm z)=\sum_{l=1}^d\sum_j\bm\alpha_{lj}\mathrm k(\bm u_l^\T\bm z,\gamma_j)$ for a random orthogonal
  $U=[\bm u_1,\dots,\bm u_d]$, a nonlinearity in a rotated coordinate system with a full-rank Jacobian.
\end{itemize}
The linear class has no translation term; translations are absorbed by the segment-wise means of the base.

\paragraph{Relation to maximum-likelihood ICA.}
Holding the base $q$ fixed, for the linear field $\vf(\bm z)=A\bm z$,
\[
  \nabla_A\Lv(0)=-\E[\s_q(\bm z)\bm z^\T+I],
\]
where the expectation is over the current pulled-back samples.
A gradient-ascent step from $A=0$ therefore gives
$A=-\alpha\,\E[\s_q(\bm z)\bm z^\T+I]$.
In the linear ICA setting $\bm z=W\x$, pulling back through
$T=\mathrm{id}+\eta\vf$ updates the unmixing matrix as
\[
  W_{\rm new}=(I+\eta A)^{-1}W
  =\bigl(I+\epsilon\,\E[I+\s_q(\bm z)\bm z^\T]\bigr)W
  +O(\epsilon^2),
  \qquad \epsilon=\eta\alpha.
\]
Thus, to first order, this recovers the relative-gradient update
for maximum-likelihood ICA with source score $\s_q$
\citep{amari1996new,cardoso1996equivariant}.
In practice, SDRE solves the regularized optimization problem \eqref{eq:sdre} to obtain the optimal refinement field, rather than taking a single gradient step from \(A=0\).

\paragraph{Gated conditional field for push-forward.} For a conditional sampler (Appendix~\ref{app:cond}) we use an affine part in
$(\bm\theta,\bm y)$ plus a local residual,
\begin{equation}
  \vf(\bm\theta,\bm y)=A\bm\theta+B\bm y+\bm b+g(\bm y)\sum_{j}\bm\alpha_j\,\mathrm k\big((\bm\theta,\bm y),\bm\gamma_j\big),\qquad
  g(\bm y)=\min\{1,\hat p_{\rm kde}(\bm y)/\tau\},
  \label{eq:cond-field}
\end{equation}
with the Gaussian kernel $\mathrm k\big((\bm\theta,\bm y),(\bm\theta',\bm y')\big)=\exp\big(-\|\bm\theta-\bm\theta'\|^2/(2\omega_\theta^2)-\|\bm y-\bm y'\|^2/(2\omega_y^2)\big)$
centered at target pairs $\bm\gamma_j$, a kernel density estimate $\hat p_{\rm kde}$ of the target observations, and $\tau$ a low quantile of $\hat p_{\rm kde}$ on them. The field is linear in $(A,B,\bm b,\bm\alpha)$, so
\eqref{eq:sdre} stays convex, and because the gate depends on $\bm y$ only, the derivatives of the local part in $\bm\theta$ are simply
scaled by $g$. The affine part carries a global correction to observations with few or no target samples; the local part is suppressed there.

\section{Push-forward on a diffusion sampler}
\label{app:push-diffusion}

Algorithm~\ref{alg:push-diffusion} interleaves the predictor steps of a pretrained deterministic diffusion sampler with calls to
$\textsc{Push}$ (Algorithm~\ref{alg:push}). At noise level $t$, the sampler should match the target noised to the same level, whose
samples $\bm\theta_{t,i}=a_t\bm\theta_i+\sigma_t\bm\xi_i$, $\bm\xi_i\sim\mathcal N(\bm0,I)$, come for free. 
The level $t_0$ represents the starting base
distribution $\pi$.

\begin{algorithm}[h]
\caption{Push-forward on a diffusion sampler}
\label{alg:push-diffusion}
\begin{algorithmic}[1]
\Require invertible predictor steps $P_1,\dots,P_S$ of a pretrained sampler, $P_s$ from noise level $t_{s-1}$ to $t_s$, with the starting base
distribution $\pi$; target pairs $\{(\bm\theta_i,\bm y_i)\}$; at most $K$ correction layers per noise level
\State $\Phi\gets\mathrm{id}$
\For{$s=1,\dots,S$}
  \State $\Phi\gets P_s\circ\Phi$ \Comment{predictor step}
  \State $\Phi\gets\textsc{Push}\big(\Phi,\{(a_{t_s}\bm\theta_i+\sigma_{t_s}\bm\xi_i,\bm y_i)\}_i\big)$ with at most $K$ layers
\EndFor
\State \Return $\Phi$; the model $\Phi_\#\pi$ has the exact log-density \eqref{eq:backtrace}
\end{algorithmic}
\end{algorithm}

\clearpage
\section{Nonlinear ICA: experiments and additional results}
\label{app:details}

This appendix gives the experimental details and comparisons for Section~\ref{sec:exp:ica}. In the tables, ``+\,\ours{}'' denotes
pull-back refinement of the pretrained model, and ``\ours{} from scratch'' starts from $\Phi_0=\mathrm{id}$.

\subsection{Data and evaluation}

In our experiments, sources are mixed by a random 4-layer leaky-ReLU network (slope $0.2$, square weight matrices with condition number
below 5), and each observed coordinate is standardized. The label is one of 20 segments, and 80\% of the samples are used for training
and 20\% for validation, on which all numbers are reported. The synthetic benchmark has $d=5$ Gaussian sources with segment-wise means
$\mu\sim U[-3,3]$ and standard deviations $\sigma\sim U[0.1,3]$, $n_c\in\{200,1000\}$ samples per segment, and 5 mixing seeds. The speech
benchmarks use real recordings as sources, resampled to 2\,kHz, 5\,s ($10^4$ samples) per source, segments of 0.25\,s, $d\in\{3,5\}$ and 3
mixing seeds: five speakers of CMU ARCTIC \citep{kominek2004cmu} and five utterances of one LibriSpeech dev-clean speaker
\citep{panayotov2015librispeech}. Speech sources are heavy-tailed and only approximately segment-wise Gaussian, so the conditional
Gaussian model is misspecified there. Identification is measured by the mean correlation coefficient (MCC) after optimal matching of
estimated and true sources.

\paragraph{Recordings and licenses.} CMU ARCTIC speakers bdl, slt, awb, clb, rms (utterances a0001--a0002 concatenated; $d=3$ uses the first three);
LibriSpeech dev-clean speaker 1272, utterances 128104-0004, 135031-0024, 141231-0002, 128104-0009, 135031-0002. Audio is resampled
from 16\,kHz to 2\,kHz with an anti-aliasing filter and the first 5\,s are used. Licenses: LibriSpeech and the Causal Chambers data
(\texttt{lt\_walks\_v1}) are released under CC BY 4.0; the CMU ARCTIC databases are distributed under a permissive free license. The sbibm task definitions are re-implemented from the sbibm benchmark, and the
iVAE baseline is a JAX port of the official implementation; both are released under the MIT license. Our code is submitted with the
paper as supplementary material.

\subsection{Models and estimation}

\paragraph{Conditional base.} We call $\textsc{Pull}$ with a conditional base
$q_{\bm\beta}(\bm z\,|\,\bm c)$ and share the field across segments. The base is either the Gaussian family
$\prod_j\mathcal N\big(z_j;\mu_j(\bm c),\sigma_j^2(\bm c)\big)$, whose parameters $\bm\beta=\{\mu_j(\bm c),\sigma_j(\bm c)\}$ are segment-wise
means and standard deviations, or a smooth super-Gaussian log-cosh family. For the log-cosh family, $\bm\beta$ instead collects segment-wise locations $\mu_j(\bm c)$ and positive scales $\sigma_j(\bm c)$, with density and score
\begin{equation*}
  q_{\bm\beta}(\bm z\mid\bm c)
  =\prod_{j=1}^d\frac{1}{\pi\sigma_j(\bm c)\cosh\!\big((z_j-\mu_j(\bm c))/\sigma_j(\bm c)\big)},
  \qquad
  [\s_{q_{\bm\beta}}(\bm z\mid\bm c)]_j
  =-\frac{1}{\sigma_j(\bm c)}\tanh\!\left(\frac{z_j-\mu_j(\bm c)}{\sigma_j(\bm c)}\right).
\end{equation*}
Here $\sigma_j(\bm c)$ is a scale parameter, not the standard deviation. Both location and scale are fitted by maximum likelihood separately for each coordinate and segment. The fitted map $\Phi$ estimates the mixing $\bm f$, up to the
indeterminacy described in Section~\ref{sec:exp:ica}, and the sources are read off as $\Phi^{-1}(\x)$. The model class is that of GIN-type flows with a conditional base
\citep{sorrenson2020disentanglement}; what differs is the estimator, a sequence of convex problems instead of end-to-end stochastic
gradient descent.

\paragraph{Restricted fields.} We restrict the transport, as Section~\ref{sec:method:stack} requires, to fields that are linear in their parameters and of a few
structured kinds: \emph{linear} fields $\vf(\bm z)=A\bm z$, \emph{coordinatewise} nonlinearities of single coordinates, their sum
(\emph{joint}), and nonlinearities along one (\emph{rank-one}) or $d$ (\emph{rank-$d$}) random directions (Appendix~\ref{app:fields}).
With linear fields, the first-order version of one step of \eqref{eq:sdre} is the relative (natural) gradient update of
maximum-likelihood ICA with source score $\s_{q_{\bm\beta}}$ \citep{amari1996new,cardoso1996equivariant}; solving \eqref{eq:sdre} exactly
instead gives a convex Fisher-scoring-type update (Appendix~\ref{app:fields}).

\paragraph{Estimators.}
We compare five standalone estimators.
The first is \emph{iVAE} \citep{khemakhem2020variational},
using our JAX port of the official implementation
(MLPs with two hidden layers of 200 units each and an output layer, $6\cdot10^3$ training steps).
The second is \emph{GIN} \citep{sorrenson2020disentanglement},
with 8 volume-preserving affine coupling blocks, a conditional
Gaussian base, and $10^4$ Adam steps.
The next two are \emph{structured flows} with the inverse architecture
of the generating network,
$W_1\chi(W_2\chi(W_3\chi(W_4\x)))$,
where $\chi$ is an elementwise invertible activation.
Both have an exact log-determinant and a conditional Gaussian base,
and are trained end to end by maximum likelihood using Adam
($5\cdot10^3$ steps; the best of four learning-rate/initialization
configurations is selected by validation likelihood).
The \emph{true-architecture flow} uses leaky ReLU with a learnable
slope, matching the activation family of the generating network.
The \emph{tanh flow} uses the smooth activation
$\chi=\mathrm{id}+\nu\tanh$, $\nu>-1$.
The main text reports the tanh flow as the
\emph{oracle-structured flow} (OSF), since its architecture uses
knowledge of the generating network.
Results for the true-architecture flow are reported in
Table~\ref{tab:true_arch}.

The fifth estimator is \ours{} trained \emph{from scratch}.
We run $\textsc{Pull}$ from $\Phi_0=\mathrm{id}$ with at most
$K=40$ layers.
At each iteration, the candidate field classes are linear,
coordinatewise, joint, and rank-one; rank-$d$ is not used in
this configuration.
The base family is conditional Gaussian for synthetic data
and conditional log-cosh for speech.

We also use \ours{} to \emph{refine} GIN and the two structured
flows, all of which have tractable exact likelihoods.
For each model, $\textsc{Pull}$ starts from its pretrained
transformation $\Phi_0$ and selects candidate layers from
all five field classes: linear, coordinatewise, joint,
rank-one, and rank-$d$.
Across all datasets, including speech, refinement uses the
conditional Gaussian base family used to train these models.
Thus, refinement differs from training \ours{} from scratch
in its initialization, candidate field classes, and,
for speech, base family.
We allow at most $K=40$ refinement layers, except for speech
with $d=5$, where $K=100$. In both configurations, the base-family parameters $\bm\beta$
are re-estimated by maximum likelihood on the current
pulled-back training samples before fitting each new layer.
Refinement starting from iVAE is discussed in
Section~\ref{app:results}.

\paragraph{\ours{} layers.} The RBF classes of Appendix~\ref{app:fields} use $L=12$ centers equally spaced on $[-5,5]$ with bandwidth $\omega=1.25$ (data are standardized).
Ridge levels $\lambda\in\{10^{-1},10^{-2},10^{-3}\}\cdot\overline{\|\bm\phi\|^2}$ (mean squared norm of the Stein features); rank-one
layers draw 15 random unit directions, keep the 3 best by the value of \eqref{eq:sdre} and refit them at two ridge levels; rank-$d$
layers draw 8 random orthogonal matrices and keep the best 2. Step sizes $\eta\in\{2,1,0.5,0.25,0.1\}$; a nonlinear candidate is rejected if,
at some training point, $\det\nabla T<\delta=0.05$ (joint, rank-one, rank-$d$) or a diagonal entry of $\nabla T$ is below $\delta$
(coordinatewise); linear candidates are not screened. Inverses: fixed-point iteration for the coordinatewise, joint and rank-one classes
(one-dimensional for rank-one), Newton's method for rank-$d$. The base family is re-estimated before every layer: segment-wise
means and standard deviations (Gaussian) or one-dimensional maximum likelihood (log-cosh). From scratch on speech, up to three
retries with fresh random directions are allowed before stopping. Convex problems are solved by L-BFGS (NumPy/SciPy) or damped
Newton (JAX); both implementations give the same layer sequence on most runs and the same table entries within $0.004$.

\paragraph{Optimization of the baselines.} iVAE: encoder/decoder MLPs with two hidden layers of 200 units each and an output layer, learning rate $10^{-2}$ with plateau decay,
$6\cdot10^3$ steps, ELBO objective. GIN: 8 affine coupling blocks with mean-zero log-scales (volume preserving), 2 hidden layers of
64 units, Adam with learning rate $10^{-3}$, batch 500, $10^4$ steps. Structured flows: Adam with cosine schedule, batch 500,
$5\cdot10^3$ steps, learning rate $\in\{10^{-2},10^{-3}\}$ and initialization $\in\{$identity$+$noise, random orthogonal$\}$,
the configuration with the best validation likelihood is reported. Extended training uses the same configuration with
$4\times$ the original number of steps.

\paragraph{Compute.} The reported \ours{} estimates used the NumPy/SciPy implementation on a laptop CPU (AMD Ryzen AI MAX+ 395, 16 cores, one core
per run). The iVAE, GIN and structured-flow baselines use JAX. The run times quoted for the JAX implementation of \ours{} were measured on a single NVIDIA RTX 4090.
Each layer solves a few dozen convex problems (field classes $\times$ ridge
levels $\times$ random directions) with at most $d^2L$ parameters on $n$ points, where $L=12$ is the number of radial basis functions per
direction (Appendix~\ref{app:fields}); with $d\le5$ this takes a fraction of a second per layer on a GPU. A full refinement of 40--100
layers runs in 6--29\,s on one RTX 4090 with our JAX implementation, or 1.5--6\,min on one CPU core.

\subsection{Main results and training controls}

Table~\ref{tab:main} gives the main comparison, and Table~\ref{tab:true_arch} adds the true-architecture flow. As a standalone estimator, \ours{} is second only to the tanh flow on synthetic data and
above the true-architecture flow on LibriSpeech with $d=3$, but it falls behind both structured flows in the other speech settings. Used
as a refiner, it improves the mean MCC of \emph{every} pretrained model in every column, and refining the tanh flow gives the best result in
all six settings, with a smaller seed-to-seed spread than the pretrained model in every speech column. Interestingly, the true-architecture flow is
not the best pretrained model to refine: extended training changes its mean MCC by at most $0.06$ (Table~\ref{tab:true_arch}), whereas refinement adds
$0.02$--$0.18$.

\begin{table}[h]\centering\small
\caption{The true-architecture structured flow (leaky-ReLU activation), complementing Tables~\ref{tab:main} and~\ref{tab:control}:
validation MCC (mean $\pm$ std over mixing seeds). Last column: seeds on which refinement beats extended training.}
\label{tab:true_arch}

\setlength{\tabcolsep}{3.7pt}
\begin{tabular}{l cccccc c}
\toprule
 & \multicolumn{2}{c}{Synthetic} & \multicolumn{2}{c}{ARCTIC} & \multicolumn{2}{c}{LibriSpeech} & \\
 & $n_c{=}200$ & $n_c{=}1000$ & $d{=}3$ & $d{=}5$ & $d{=}3$ & $d{=}5$ & wins \\
\midrule
Structured flow (true arch.) & 0.591{\scriptsize$\pm$0.071} & 0.574{\scriptsize$\pm$0.061} & 0.767{\scriptsize$\pm$0.144} & 0.614{\scriptsize$\pm$0.013} & 0.750{\scriptsize$\pm$0.139} & 0.651{\scriptsize$\pm$0.075} & \\
\quad Extended training & 0.604{\scriptsize$\pm$0.074} & 0.595{\scriptsize$\pm$0.072} & 0.759{\scriptsize$\pm$0.134} & 0.670{\scriptsize$\pm$0.095} & 0.756{\scriptsize$\pm$0.143} & 0.673{\scriptsize$\pm$0.065} & \\
\quad + \ours{} & \textbf{0.768}{\scriptsize$\pm$0.076} & \textbf{0.734}{\scriptsize$\pm$0.073} & \textbf{0.783}{\scriptsize$\pm$0.166} & \textbf{0.704}{\scriptsize$\pm$0.074} & \textbf{0.792}{\scriptsize$\pm$0.160} & \textbf{0.712}{\scriptsize$\pm$0.123} & 19/22 \\
\bottomrule
\end{tabular}
\end{table}

Refinement adds further training, so a natural objection is that the pretrained model was merely under-trained. To test this, we use extended training
for every pretrained model (Tables~\ref{tab:control} and~\ref{tab:true_arch}). Extended training changes the mean MCC by $-0.022$ to $+0.094$;
even its largest gain ($0.762\to0.856$ for the tanh flow on synthetic $n_c=1000$) stays below refinement ($0.887$). Overall, refinement
is better on 58 of 66 seeds over the three pretrained models (39 of 44 in Table~\ref{tab:control}, 19 of 22 for the true-architecture flow), and on speech extended training can even reduce MCC (4 of the 12 speech entries).

\subsection{Additional analyses}
\label{app:results}

\paragraph{Starting from iVAE.} Starting from the encoder mean of iVAE, which is neither invertible nor equipped with a
log-determinant, changes MCC by at most $0.025$ on the synthetic benchmark ($0.484\to0.509$ and $0.492\to0.498$) and by $-0.028$ to
$+0.048$ on speech (ARCTIC $0.538\to0.586$ and $0.383\to0.367$, LibriSpeech $0.578\to0.560$ and $0.461\to0.432$ for $d=3,5$). It improves
12 of the 22 seeds, so there is no systematic gain.

\paragraph{Stopping by the discrepancy.} Can we use $\hat J$ to decide when to stop? The value $\hat J$ of each accepted layer estimates a lower bound on the remaining divergence.
Computed on the samples used to fit the layer, it is about $m/(2n)$ for a layer with $m$ parameters even when the pulled-back samples
already follow the base, which suggests stopping once $\hat J<\kappa\,m/(2n)$ for a constant $\kappa\ge1$. On 42 traced refinements
(Table~\ref{tab:stop}), this rule is counter-productive: running all 40 layers gives mean MCC $0.770$,
stopping with $\kappa=1$ or $4$ gives $0.736$ or $0.717$, and even an oracle that stops at the best layer only reaches $0.777$. Later
layers still help although each is individually insignificant. This does not argue against the test in push-forward algorithm: there the pretrained
model is trusted and every change must be justified by a detected discrepancy, whereas here the pretrained model is known to be imperfect.

\begin{table}[h]\centering\small
\caption{Stopping refinement by $\hat J<\kappa\,m/(2n)$, post hoc on 42 traced 40-layer refinements (ARCTIC and LibriSpeech
$d=3$, ARCTIC 20\,s $d=3$, synthetic $n_c=200$). Mean validation MCC at the stopping layer (mean number of accepted layers).}
\label{tab:stop}
\begin{tabular}{lcccccc}
\toprule
Pretrained model & no refinement & all 40 layers & oracle layer & $\kappa=1$ & $\kappa=4$ & $\kappa=16$ \\
\midrule
all & 0.706 & 0.770 & 0.777 & 0.736 (15.2) & 0.717 (5.2) & 0.705 (1.4) \\
GIN & 0.589 & 0.662 & 0.664 & 0.630 (8.6) & 0.617 (3.4) & 0.595 (0.4) \\
Structured (true) & 0.706 & 0.769 & 0.788 & 0.726 (23.7) & 0.694 (7.4) & 0.691 (2.4) \\
Structured (tanh) & 0.824 & 0.878 & 0.879 & 0.851 (13.4) & 0.841 (4.9) & 0.829 (1.3) \\
\bottomrule
\end{tabular}
\end{table}

\paragraph{Selecting among refined pretrained models.} In the following experiments, final validation likelihood is effective for choosing among refined pretrained models. Over the 22 seeds of Table~\ref{tab:main}
and the six seeds of the 20\,s ARCTIC setting below, refining every pretrained model and keeping the one with the highest final validation
likelihood gives mean MCC $0.857$, against $0.855$ for always refining the tanh flow, and is never worse than the latter on any seed.
Selecting the pretrained model by its likelihood \emph{before} refinement instead picks the flexible GIN in 7 of 28 cases and drops to $0.810$. 

\paragraph{Longer recordings and a failure of refinement.} More data do not always help, and refinement can fail. With 20\,s instead of 5\,s per source (seven ARCTIC speakers, with
speaker subsets varying across seeds), MCC does not improve systematically (Table~\ref{tab:arctic20}): only
iVAE and the unrefined true-architecture flow gain in both dimensions, and the comparison is confounded by the change of speakers. The
refined tanh flow remains best ($0.865$ and $0.791$ for $d=3,5$) and beats extended training on all 6 seeds. Refining the
true-architecture flow, however, beats the control on only 2 of 6 seeds, and on one seed MCC collapses from $0.698$ to $0.502$ while the
validation likelihood rises from $-1.34$ to $+0.90$. From a poor pretrained model, the likelihood itself can lead away from the sources, which no
likelihood-based signal, including $\hat J$, can detect. This is the caveat of Section~\ref{sec:method:stack} in practice: the likelihood
of a transformation model does not by itself say which transformation recovers the sources.

\begin{table}[h]\centering\small
\caption{ARCTIC with 5\,s versus 20\,s per source (validation MCC, mean $\pm$ std over 3 seeds). The 20\,s setting uses seven
speakers with a different speaker subset per seed. $^\dagger$One iVAE run diverged; mean over 2 seeds.}
\label{tab:arctic20}
\begin{tabular}{lcccc}
\toprule
 & \multicolumn{2}{c}{$d=3$} & \multicolumn{2}{c}{$d=5$} \\
\cmidrule(lr){2-3}\cmidrule(lr){4-5}
Method & 5\,s & 20\,s & 5\,s & 20\,s \\
\midrule
iVAE & 0.538{\scriptsize$\pm$0.090} & 0.573{\scriptsize$\pm$0.033}$^\dagger$ & 0.383{\scriptsize$\pm$0.034} & 0.494{\scriptsize$\pm$0.014} \\
GIN & 0.645{\scriptsize$\pm$0.084} & 0.525{\scriptsize$\pm$0.057} & 0.440{\scriptsize$\pm$0.037} & 0.432{\scriptsize$\pm$0.019} \\
Structured flow (true arch.) & 0.767{\scriptsize$\pm$0.144} & 0.792{\scriptsize$\pm$0.073} & 0.614{\scriptsize$\pm$0.013} & 0.627{\scriptsize$\pm$0.058} \\
Structured flow (tanh) & 0.854{\scriptsize$\pm$0.024} & 0.816{\scriptsize$\pm$0.017} & 0.747{\scriptsize$\pm$0.046} & 0.711{\scriptsize$\pm$0.062} \\
\ours{} (from scratch) & 0.621{\scriptsize$\pm$0.150} & 0.598{\scriptsize$\pm$0.057} & 0.518{\scriptsize$\pm$0.047} & 0.480{\scriptsize$\pm$0.007} \\
\midrule
GIN + \ours{} & 0.747{\scriptsize$\pm$0.090} & 0.655{\scriptsize$\pm$0.104} & 0.587{\scriptsize$\pm$0.081} & 0.594{\scriptsize$\pm$0.083} \\
Structured (true) + \ours{} & 0.783{\scriptsize$\pm$0.166} & 0.736{\scriptsize$\pm$0.167} & 0.704{\scriptsize$\pm$0.074} & 0.680{\scriptsize$\pm$0.007} \\
Structured (tanh) + \ours{} & \textbf{0.883}{\scriptsize$\pm$0.007} & \textbf{0.865}{\scriptsize$\pm$0.008} & \textbf{0.805}{\scriptsize$\pm$0.031} & \textbf{0.791}{\scriptsize$\pm$0.031} \\
\bottomrule
\end{tabular}
\end{table}

\paragraph{Physical mixing.} Real physical mixing remains hard. On the light tunnel of the Causal Chambers \citep{gamella2025causal}, the actuators (LED intensities and
a polarizer angle) are mixed by physics and recorded by photodiodes. All nonlinear methods, including ours, stay below linear FastICA
\citep{hyvarinen1999fast} (MCC $0.878$ and $0.815$ on two experiments, against at most $0.798$ and $0.623$; Table~\ref{tab:lighttunnel}), in
line with \citet{gamella2025sanity}. The sources are strongly sub-Gaussian within segments (excess kurtosis $\approx-1.2$) and their
distribution changes little across segments, so the Gaussian and super-Gaussian base families are misspecified in the direction
that matters for ICA \citep{amari1997stability}; a base family that can be sub-Gaussian is needed here.

\begin{table}[h]\centering\small
\caption{Causal Chambers light tunnel (\texttt{lt\_walks\_v1}): photodiode readings projected on their top three principal components
(the six sensors come in near-duplicate pairs), 20 time segments, validation MCC against the recoverable actuators, mean $\pm$ std over
3 train/validation splits. Targets: red, blue and $\cos^2$ of the polarizer angle (\texttt{smooth\_polarizers}); red, green, blue
(\texttt{color\_mix}). The supervised bound regresses each target on the inputs with random Fourier features.}
\label{tab:lighttunnel}
\begin{tabular}{lcc}
\toprule
Method & \texttt{smooth\_polarizers} & \texttt{color\_mix} \\
\midrule
Supervised bound & 0.946{\scriptsize$\pm$0.002} & 0.949{\scriptsize$\pm$0.008} \\
Principal components (no unmixing) & 0.754{\scriptsize$\pm$0.002} & 0.582{\scriptsize$\pm$0.013} \\
FastICA (linear) & \textbf{0.878}{\scriptsize$\pm$0.001} & \textbf{0.815}{\scriptsize$\pm$0.006} \\
\midrule
iVAE & 0.798{\scriptsize$\pm$0.034} & 0.583{\scriptsize$\pm$0.016} \\
GIN & 0.638{\scriptsize$\pm$0.001} & 0.533{\scriptsize$\pm$0.026} \\
Structured flow (true arch.) & 0.665{\scriptsize$\pm$0.004} & 0.615{\scriptsize$\pm$0.006} \\
Structured flow (tanh) & 0.649{\scriptsize$\pm$0.016} & 0.623{\scriptsize$\pm$0.012} \\
\ours{} from scratch (Gaussian / log-cosh) & 0.772 / 0.677 & 0.553 / 0.546 \\
GIN / true / tanh + \ours{} & 0.642 / 0.649 / 0.647 & 0.522 / 0.583 / 0.598 \\
\bottomrule
\end{tabular}
\end{table}

\paragraph{Stock returns.} We use the MSFT, IBM, SBUX, AAPL and S\&P\,500 price series in Plotly's public \texttt{stockdata.csv}.\footnote{\url{https://github.com/plotly/datasets/blob/master/stockdata.csv}} We take first differences of the log prices as supplied, with no additional price adjustment, and discard the final incomplete quarter. This leaves 2,265 daily returns from 4 January 2007 to 31 December 2015, grouped into 36 calendar-quarter segments. Each return series is standardized, and the observations are randomly split into 80\% training and 20\% validation data. With a log-cosh base family, 23 \ours{} layers reached a validation
log-likelihood of $-4.83$ per day, compared with $-4.92$ and $-4.98$ for four-layer structured flows with a tanh and a leaky-ReLU
activation (best of four learning-rate and initialization configurations by validation likelihood), $-6.27$ for GIN (which overfits:
$-4.43$ on the training data) and $-6.31$ for independent marginals. The recovered components are one market factor (correlation
$0.99$ with the S\&P\,500) and four idiosyncratic ones, each correlated with one stock ($0.69$--$0.79$) and almost uncorrelated
with the index ($\le0.09$).

\clearpage
\section{Simulation-based inference: experiments and comparisons}
\label{app:sbi}

In Table~\ref{tab:sbi},
``+\,\ours{}'' denotes push-forward correction of the frozen posterior sampler.

\subsection{Tasks and target samples}

\paragraph{Tasks and references.} Both sbibm tasks \citep{lueckmann2021benchmarking} have 10-dimensional parameters and observations.
\texttt{gaussian\_linear}: $\bm\theta\sim\mathcal N(\bm0,0.1I_{10})$, $\bm y=\bm\theta+\mathcal N(\bm0,0.1I_{10})$;
all posteriors (including those under mixture priors) are analytic. \texttt{bernoulli\_glm}: bias and nine filter weights, a fixed
white-noise stimulus of 100 time bins, Bernoulli spikes with logit $\psi_j=b+\bm w^\T\bm u_j$ in time bin $j$, summary statistics the spike count and the
spike-triggered sum; prior as in sbibm. In the modified target simulator used for the simulator-change and partial-coverage experiments, the spike probability
is $\mathrm{sigmoid}(1.3\,\psi_j)$ instead of $\mathrm{sigmoid}(\psi_j)$. Reference posteriors are drawn by Hamiltonian Monte Carlo preconditioned with the Hessian at the
MAP (64 chains, 10 leapfrog steps, 2000 warm-up iterations, thinning 5; acceptance $0.99$); two independent runs give C2ST $0.49$. The
sufficient-statistic likelihood was checked against the raw Bernoulli likelihood. New priors: $\mathcal N(0.15\cdot\bm1,0.05I)$ and the
mixture $\frac{1}{2}\mathcal N(\pm0.2\cdot\bm1,0.05I)$ (\texttt{gaussian\_linear}); bias $\mathcal N(-1,1)$ and the filter precision doubled
(\texttt{bernoulli\_glm}).

For prior changes, target pairs are resampled with replacement from the existing $10^5$ simulations, with probabilities
proportional to the prior ratio, without new simulations. For simulator change, we generate $n_{\rm real}\in\{500,2000\}$ target pairs
from the modified simulator. For partial coverage, we generate $10^3$ pairs from that simulator restricted to neurons with negative bias;
evaluation covers both negative-bias (covered) and positive-bias (uncovered) neurons.

All evaluated corrections use the same target pairs as \ours{}; Table~\ref{tab:sbi} specifies which settings are evaluated for each baseline.

\subsection{Sampler and corrections}

\paragraph{Pretrained sampler.} On each task, we train NPSE \citep{song2021score,sharrock2024sequential} on $10^5$ simulations
and freeze its network during correction. We use VP diffusion (linear noise schedule from $0.1$ to $20$) with a noise-prediction MLP (4 hidden layers of 512 SiLU units, sinusoidal
time embedding, skip term $\sigma_t\bm\theta_t$ for the noised parameter $\bm\theta_t$) on standardized parameters and observations; denoising score matching for $t\in[10^{-3},1]$,
$3\cdot10^4$ Adam steps with batch 1024, exponential moving average $0.999$ and early stopping on validation simulations. Sampling uses 25 deterministic DDIM steps \citep{song2021denoising} on the grid
$t_s=10^{-3}+(1-10^{-3})(1-s/25)^2$, $s=0,\dots,25$, so that $t_0=1$ is the level of the latent distribution. Training is deterministic given the seed.

\paragraph{\ours{} layers.} We use the gated field \eqref{eq:cond-field} and the acceptance rule of Algorithm~\ref{alg:push},
with hyperparameters selected on validation target samples. The field has 300 centers $\bm\gamma_j$ drawn from the noised target pairs, one
bandwidth pair $\omega_\theta\in\{1,2\}$, $\omega_y\in\{1.5,3\}$ (standardized units), and $\tau$ the 5\% quantile of $\hat p_{\rm kde}$
on the target observations. Unlike in \eqref{eq:sdre}, the penalty is $\lambda\cdot\frac{1}n\sum_i\|\vf(\bm\theta_i,\bm y_i)\|^2$ (no factor
$\frac{1}{2}$), and the bandwidth pair and $\lambda\in\{10^{-3},10^{-4}\}$ are chosen jointly by the validation $\hat J$ at every layer (all
$4\times2$ combinations are fitted). L-BFGS is stopped at the best validation value of the objective: with about $3\cdot10^3$
coefficients and these small $\lambda$, the empirical objective keeps decreasing long after the validation one starts to increase. Step
$\eta_k=\min\{1,1/(2\max_i\|\nabla\hat\vf_k\|_2)\}$ with the spectral norm and the maximum over the target samples; up to $K=6$ layers
per noise level; 80\% of the target samples fit the field and 20\% decide acceptance. Target sample sizes: 3000 simulations drawn with replacement
with probabilities proportional to the prior ratio (prior change), and the training part of the $n_{\rm real}$ target pairs, each noised several times to about 3000 points per level
(simulator change and partial coverage).

\subsection{Evaluation and baselines}

We measure accuracy by the classifier two-sample test
\citep[C2ST;][]{lopezpaz2017revisiting}: the accuracy of a classifier trained to separate $5\cdot10^3$ samples of the sampler from
reference samples, with chance accuracy $0.5$ and lower values indicating a closer match. It is averaged over 10 observations drawn from the target (20 for partial coverage, 10 in each half),
with three seeds for the network, the target samples and the push-forward layers.

\paragraph{Baselines.} PriorGuide \citep{yang2025priorguide} is evaluated only for prior changes. The prior ratio is one Gaussian (unimodal and \texttt{bernoulli\_glm}) or a two-component mixture, so
its mixture representation is exact; reverse kernel $\mathcal N(\hat{\bm\theta},\sigma_t^2/(a_t^2+\sigma_t^2)I)$ with the Tweedie estimate
$\hat{\bm\theta}$ of the clean parameter; the gradient of the guidance term is computed exactly through $\hat{\bm\theta}$, so the guided ODE is an invertible
sampler with a tractable density. Langevin refinement: $\{2,5,10\}$ steps per level, signal-to-noise ratio $\{0.08,0.16\}$, best by the
test C2ST. DG \citep{kim2023refining}: 25 discriminators (4 hidden layers of 512 units, $8\cdot10^3$ Adam steps, learning rate $3\cdot10^{-4}$, fresh forward noise every 25
steps) between noised target pairs and noised model pairs with the same observations, model samples from a 100-step DDIM sampler;
guidance adds the discriminator-logit gradient to the score with weight 1, followed by Langevin refinement on the same grid,
selected by the maximum mean discrepancy on validation pairs (selection by the test C2ST changes the numbers by
at most $0.006$). Classifier DRE + SIR \citep{bickel2009discriminative}: MLPs with $2\times256$ and $3\times512$ units between the target pairs and as many simulated pairs,
the one with the lowest validation log-loss; SIR on $5\cdot10^3$ samples of the frozen sampler. The simulated pairs are a uniform
subsample of the existing simulations under a prior change, so that no new simulation is run, fresh simulations under a simulator change,
and fresh simulations from the same half of the prior (negative bias) under partial coverage, so that the classifier estimates the change
of likelihood on the covered half. Fine-tuning on the target pairs: $4\cdot10^3$ Adam steps at learning
rate $3\cdot10^{-4}$ from the pretrained weights with early stopping on validation pairs; the sim.\ + real variant is evaluated only under partial coverage,
as a rehearsal baseline for the uncovered region. It mixes $2\cdot10^4$ original simulations with
oversampled target pairs in equal proportion. The retrained reference uses the same architecture, trained from scratch on $10^5$
new target simulations, beyond the correction budget; under partial coverage, it is the network retrained for
the simulator change, evaluated on the partial-coverage observations.

\subsection{Statistical comparisons and limitations}

For Table~\ref{tab:sbi}, paired $t$-tests over seeds compare \ours{} with PriorGuide + Langevin under prior change
($p=0.092$, $0.003$ and $0.56$ for unimodal, bimodal and GLM, respectively). PriorGuide + Langevin has lower mean C2ST in all three
settings, with a statistically significant difference only for the bimodal prior; its hyperparameters are selected by test C2ST.
Under simulator change, the comparisons with DG give $p=0.42$ and $0.43$, while \ours{} outperforms fine-tuning with $p=0.004$ and
$0.009$ for $n_{\rm real}=500$ and $2000$, respectively. Under partial coverage, \ours{} outperforms fine-tuning on both the covered
and uncovered regions ($p=0.020$ and $0.024$, respectively).

\paragraph{Separated modes.} Push-forward only transports the samples that the sampler already produces. When the target
needs mass moved between separated modes, as for the bimodal prior in Table~\ref{tab:sbi} or for a pretrained sampler that misses a mode
altogether, each layer can only make a small, invertible, locally Lipschitz move. The per-layer discrepancy then quickly falls below the
detection threshold, and push-forward stops short (C2ST $0.573$ against $0.510$ for retraining); a stochastic correction that is not
limited to small invertible moves, such as the Langevin refinement of PriorGuide, is better suited to this case.

\paragraph{Compute.} Between two and eleven correction layers were accepted per task. One GH200. Per run: \ours{} 3--4\,min, DG 9--13\,min (dominated by the discriminators), PriorGuide with Langevin
4--7\,min, fine-tuning under 1\,min, retraining 2\,min.

\section{Related Work}
\label{sec:related}

Direct density ratio estimation fits the ratio without modeling either density: moment matching in an RKHS \citep{huang2007correcting}, KLIEP and its
log-linear form \citep{sugiyama2008direct,tsuboi2009direct}, least-squares importance fitting \citep{kanamori2009least}, convex risk
minimization for $f$-divergences \citep{nguyen2010estimating}, and probabilistic classification
\citep{bickel2009discriminative,gutmann2010noise}; see \citet{sugiyama2012density} for a monograph. Neural variational bounds
\citep{belghazi2018mine} use the same DV representation as \eqref{eq:stein-dv}. When the two distributions are far apart, telescoping
DRE chains classifiers along fixed bridging distributions \citep{rhodes2020telescoping}, \citet{choi2022density} take the infinitesimal
limit of such bridges, and featurized DRE estimates the ratio in the latent space of a shared flow \citep{choi2021featurized}. The
pull-back estimator is also a product of small ratios, but the intermediate distributions are produced by transport and adapted to the
data, and the final map is itself a flow, as in featurized DRE. All of these methods need samples from both distributions; we need
samples from one and the score of the other. Stein-parametrized ratio models were introduced for fitting intractable models
\citep{liu2019fisher}; we add the displacement interpretation, the connection to the Fisher--Rao flow, and the push-forward and pull-back
algorithms. Score-based alternatives such as score matching \citep{hyvarinen2005estimation} and Stein discrepancies
\citep{gorham2015measuring,liu2016kernelized,chwialkowski2016kernel} compare a model with data through the model score, but return a
discrepancy rather than a ratio or a transport.

Among gradient flows, Wasserstein gradient flows of the KL divergence \citep{jordan1998variational}, and their particle approximation SVGD
\citep{liu2016stein}, need the score of the target. Fisher--Rao and Wasserstein--Fisher--Rao flows
\citep{chizat2018interpolating,liero2018optimal} are simulated by birth--death dynamics \citep{lu2019accelerating} or by kernel-based
transport \citep{maurais2024sampling,nusken2024stein}, which we compare with next.

The closest works to \ours{} are kernel Fisher--Rao flow \citep{maurais2024sampling} and Stein transport \citep{nusken2024stein}. Both
move particles along a tempered path from a reference $\pi_0$ to a target $\pi_1$, $\pi_t\propto\pi_0^{1-t}\pi_1^{t}$, with a velocity field
that solves the same equation as our optimal field: the continuity equation $-\nabla\cdot(\pi_t\vf_t)=\pi_t\big(\log\ell-\E_{\pi_t}[\log\ell]\big)$
with $\ell=\pi_1/\pi_0$, which in Stein form reads $-\A_{\pi_t}\vf_t=\log\ell-\E_{\pi_t}[\log\ell]$ (compare Proposition~\ref{prop:ratio}).
Kernel Fisher--Rao flow takes the solution of minimal kinetic energy, a gradient field obtained from a Poisson equation in an RKHS, and
needs only evaluations of $\ell$; Stein transport fits the Stein equation by kernel ridge regression and also uses the scores along the
path. In both, the log-ratio $\log\ell$ is known (in Bayesian inference, it is the log-likelihood) and serves as the right-hand side of the
equation. In our setting the target is observed only through samples, so the log-ratio is unknown: \ours{} estimates it together with the
field. Moreover, both methods traverse a
prescribed path in unit time, whereas push-forward and pull-back take one estimated step at a time and re-estimate the ratio after each.

Pulling the data back, one simple map at a time, until they look like a simple base distribution is a classic idea. Projection pursuit
removes structure along one direction at a time by Gaussianizing the projection \citep{friedman1987exploratory}; it is the counterpart
of our rank-one class. Gaussianization alternates linear ICA with marginal Gaussianization \citep{chen2001gaussianization}, the
counterpart of our linear and coordinatewise classes, and RBIG replaces the ICA step by random rotations and shows convergence
\citep{laparra2011iterative}, the counterpart of our rank-$d$ class. \citet{tabak2010density} and \citet{tabak2013family}, whose work introduced the term
\emph{normalizing flow}, compose simple maps that each increase the likelihood of the pulled-back samples under a standard normal
base. As shown in Appendix~\ref{app:alg:pull}, their ascent
direction is the linear Stein functional of which \eqref{eq:stein-dv} is the log-moment counterpart. Density destructors
\citep{inouye2018deep}, sliced iterative normalizing flows \citep{dai2021sliced} and gradient-boosted flows
\citep{giaquinto2020gradient} develop the greedy construction further, and Gaussianization flows make iterative Gaussianization
trainable end to end \citep{meng2020gaussianization}. Our construction
differs in three ways: each layer is the solution of a convex problem whose value estimates a lower bound on the remaining divergence; the
base is a parametric family re-estimated at every step and may be conditional on an auxiliary variable; and every step is at the
same time a density ratio estimate. See \citet{papamakarios2021normalizing} for normalizing flows in general.

In independent component analysis, linear ICA \citep{comon1994independent} is identified by non-Gaussianity; maximum-likelihood and infomax estimators
\citep{bell1995information,cardoso1997infomax,pham1997blind} are optimized by the relative gradient \citep{amari1996new}, whose
stability depends on matching the sign of the source kurtosis \citep{amari1997stability}; FastICA \citep{hyvarinen1999fast} adapts to
this sign. Nonlinear ICA is not identifiable without further assumptions \citep{hyvarinen1999nonlinear}; post-nonlinear mixtures
\citep{taleb1999source} are a classical exception. Auxiliary variables restore identifiability: time-contrastive learning
\citep{hyvarinen2016unsupervised}, temporal dependence \citep{hyvarinen2017nonlinear}, generalized contrastive learning
\citep{hyvarinen2019nonlinear}, iVAE \citep{khemakhem2020variational}, GIN \citep{sorrenson2020disentanglement} and hidden-Markov
variants \citep{halva2020hidden}; other routes constrain the mixing \citep{gresele2021independent} or use contrastive objectives
\citep{zimmermann2021contrastive}. We use the same model class as GIN---a conditional base and an exact-likelihood flow---but
replace end-to-end training by convex layer-wise estimation, and use it to refine end-to-end estimators. Real-world checks of these
methods remain sobering \citep{gamella2025sanity}, as our light-tunnel results confirm.

\end{document}